\documentclass[11pt]{article}
\usepackage{fullpage}
\usepackage{amsmath,amsthm,amsfonts,amssymb,dsfont}
\usepackage{microtype}
\usepackage{hyperref,cleveref}

\Crefformat{equation}{#2(#1)#3}
\crefformat{equation}{#2(#1)#3}
\Crefrangeformat{equation}{#3(#1)#4--#5(#2)#6}
\crefrangeformat{equation}{#3(#1)#4--#5(#2)#6}
\Crefmultiformat{equation}{#2(#1)#3}{ and~#2(#1)#3}{, #2(#1)#3}{ and~#2(#1)#3}
\crefmultiformat{equation}{#2(#1)#3}{ and~#2(#1)#3}{, #2(#1)#3}{ and~#2(#1)#3}

\usepackage{enumitem}
\setlist[description]{font=\normalfont}

\usepackage{csquotes}
\MakeOuterQuote{"}

\usepackage{mathtools}
\mathtoolsset{centercolon}
\DeclarePairedDelimiterXPP\ind[1]{\mathds{1}}{\lbrace}{\rbrace}{}{#1} %
\DeclarePairedDelimiter\ip{\langle}{\rangle} %
\DeclarePairedDelimiter\abs{\lvert}{\rvert} %
\DeclarePairedDelimiter\card{\lvert}{\rvert} %
\DeclarePairedDelimiter\norm{\lVert}{\rVert} %
\DeclarePairedDelimiter\ceil{\lceil}{\rceil} %
\DeclarePairedDelimiter\floor{\lfloor}{\rfloor} %
\DeclarePairedDelimiter\del{\lparen}{\rparen} %
\DeclarePairedDelimiter\sbr{\lbrack}{\rbrack} %
\DeclarePairedDelimiter\set{\lbrace}{\rbrace} %
\DeclarePairedDelimiter\intoo{\lparen}{\rparen} %
\DeclarePairedDelimiter\intcc{\lbrack}{\rbrack} %

\providecommand\given{}
\newcommand{\pipeseparator}{\nonscript\:\delimsize\vert\nonscript\:\mathopen{}}
\begingroup
  \catcode`\|=13
  \gdef|{\pipeseparator}
\endgroup
\newcommand{\activatepipe}{%
  \renewcommand\given\pipeseparator
  \mathcode`\|="8000
}

\DeclarePairedDelimiterX{\Set}[1]{\{}{\}}{%
  \activatepipe
  #1
}

\DeclareMathOperator\sign{sign} %
\DeclareMathOperator\Span{span} %
\DeclareMathOperator\Conv{conv} %
\DeclareMathOperator\argmax{\arg\max} %

\newcommand\T{{\scriptscriptstyle{\mathsf{T}}}} %
\newcommand\R{\mathbb{R}} %
\newcommand\N{\mathbb{N}} %
\newcommand\Z{\mathbb{Z}} %

\newcommand\calB{\mathcal{B}}

\DeclareMathOperator\Att{Att}
\DeclareMathOperator\Argmax{Arg\max}
\DeclareMathOperator\poly{poly}

\usepackage[numbers,sort]{natbib}

\usepackage{thmtools}
\usepackage{thm-restate}

\declaretheorem[name=Theorem]{theorem}
\declaretheorem[name=Proposition]{proposition}
\declaretheorem[name=Lemma]{lemma}

\declaretheorem[name=Corollary]{corollary}
\crefname{claim}{claim}{claims}

\newcommand{\iu}{\mathrm{i}}

\title{Attention-based representations for multi-task computation}
\author{Daniel Hsu\thanks{%
    \texttt{djhsu@cs.columbia.edu}.
    Supported in part by the National Science Foundation under grant DMS-2502259, the Office of Naval Research under grant N00014-24-1-2700, and a Research Award from the Columbia Center of AI Technology in collaboration with Amazon.%
  } \\
  {\sl Columbia University}
  \and
  Mingyue Xu\thanks{%
    \texttt{xu1864@purdue.edu}.%
  } \\
  {\sl Purdue University}
}

\begin{document}

\maketitle

\begin{abstract}
  Multi-head attention layers produce vector representations that support multiple downstream tasks.
  We establish bounds on the number of heads required in two simple and concrete multi-task scenarios.
  In the first scenario, a vector representation is sought so that linear predictors can compute both the smallest and largest numbers in a given list.
  In this case, it is known two attention heads with small embedding dimension and bit precision level suffice.
  We prove that a single attention head requires exponentially higher embedding dimension or precision level.
  In the second scenario, a vector representation is sought so that a polynomial threshold function can compute the XOR of a given string of $n$ bits.
  This scenario is analogous to the first one for $n=2$, since XOR is readily computed by a linear function using a vector representation that encodes both the AND and the OR of the two bits.
  We observe that $n$-bit XOR requires the product of the number of heads and the polynomial degree to be at least $n$, and we construct multi-head attention layers that match this lower bound.
  These results generalize to arbitrary (symmetric) Boolean functions, where the bound is given in terms of the threshold degree.
\end{abstract}

\section{Introduction}

Attention heads---which are the building blocks of transformers~\citep{vaswani2017attentionneed}---have been successfully used to learn "task-independent" representations in a variety of scenarios~\citep[e.g.,][]{lin2017structuredselfattentivesentenceembedding,parikh2016decomposableattentionmodelnatural}.
In transformers, layers of several attention heads are composed together to form these representations, and the use of \emph{multiple layers} seems to be necessary for several natural computational tasks related to multi-step reasoning~\citep{merrill2023parallelismtradeofflimitationslogprecision,chiang2023tighterboundsexpressivitytransformer,liu2023transformerslearnshortcutsautomata,sanford2024transformersparallelcomputationlogarithmic,chen2024theoreticallimitationsmultilayertransformer,merrill2025littledepthgoeslong}.

What is less clearly established in the literature is the necessity of \emph{multiple attention heads in a single layer}.
(Recent exceptions are discussed in \Cref{sec:related}.)
Intuitively speaking, using multiple attention heads allows for multiple computations to be carried out in parallel, and hence produce representations that are useful for multiple tasks.
However, the necessity of using multiple heads to achieve this is not a foregone conclusion.
Indeed, a single attention head---with output elements post-processed by a neural network---can approximate essentially any function arbitrarily well, as long as the attention head's embedding dimension or bit precision level is high enough~\citep[e.g.,][]{kajitsuka2024transformerslayerselfattentionusing}.
So the necessity of multiple heads must arise from restrictions on aspects such as dimension, precision level, and the post-processing.

This article considers the limits of attention heads for producing vector representations that support multiple tasks in two scenarios.
In the first scenario, the two tasks are computing the minimum and the maximum of a given list of integers.
Separately, each task is easily and succinctly supported by a single attention head, post-processed by a linear classifier.
We show, in \Cref{thm:minmax}, that a single attention head cannot support both tasks simultaneously in this way unless the value embedding dimension or precision level is large as a function of the input or problem domain size.

In the second scenario, the overall task is to compute the exclusive-or (XOR) of a given string of $n$ bits.
For $n=2$ bits, the XOR is readily determined as a linear function of two other binary features: the AND of the bits and the OR of the bits.
Separately, each of AND and OR is easily and succinctly handled by a single attention head, post-processed by a linear classifier; in this sense, the $2$-bit XOR scenario is similar to minimum and maximum computation.
\citet{viswanathan2026xor} recently observed that no single attention head post-processed by a linear classifier can compute $2$-bit XOR.
We generalize this observation to $n$-bit XOR: any attention layer post-processed by a polynomial threshold function that computes $n$-bit XOR must satisfy
\begin{equation*}
  \text{polynomial degree} \times \text{number of attention heads} \geq n .
\end{equation*}
We also show that the bound is tight: for any positive integers $D$ and $H$ satisfying $D \times H \geq n$, $n$-bit XOR can be computed by an $H$-head attention layer (with value embedding dimension at most $O(\min\set{D,\log H})$), post-processed by a degree-$D$ polynomial threshold function.
Our results for $n$-bit XOR generalize to arbitrary (symmetric) Boolean functions with threshold degree $T$; we show that $D \times H \geq T$ is both necessary and sufficient.

\subsection{Related works}
\label{sec:related}

Many recent theoretical results about transformers emphasize the importance of multiple attention layers for certain tasks~\citep[e.g.,][]{merrill2023parallelismtradeofflimitationslogprecision,chiang2023tighterboundsexpressivitytransformer,liu2023transformerslearnshortcutsautomata,sanford2023representationalstrengthslimitationstransformers,peng2024limitationstransformerarchitecture,sanford2024transformersparallelcomputationlogarithmic,chen2024theoreticallimitationsmultilayertransformer,kozachinskiy2024lowerboundstransformersinfinite,sanford2024onelayertransformersfailsolve,merrill2025littledepthgoeslong,kozachinskiy2026paritysensitivitytransformers}.
Our focus is the importance of multiple heads in a single layer.

A few recent works address the role of multiple attention heads in a single layer.
\citet{yu2026effectattentionheadcount} proves separations between $H$ and $H+1$ attention heads, for all $H\geq1$; here we only compare to the specialization of their result to $H=1$ heads.
Their lower bounds are established for a family of tasks of approximately computing functions similar to $(x_1,\dotsc,x_n) \mapsto \min \set{x_1,\dotsc,x_n} + \max \set{x_1,\dotsc,x_n}$ over a bounded real domain; these tasks are similar in spirit to the problem we study.
They show that when the query, key, and value embeddings have uniformly bounded norms, and the embedding dimension is much smaller than $n$, then by a pigeonhole principle argument, there are at least two inputs (with different correct outputs) that are mapped to similar vectors by the attention head, and such vectors are not well-distinguished by Lipschitz neural networks.
Our result on simultaneous $\min$ and $\max$ computation is not directly comparable, in that
(i) we consider exact computation of $\min$ and $\max$ over a finite domain (as opposed to approximate computation of real-valued functions);
(ii) we only consider a single attention head (as opposed to separations for larger numbers of heads);
(iii) we only consider post-processing by linear classifiers (as opposed to more general neural networks); and
(iv) we put precision limits on the value embeddings.
However, in our result, (i) we do not require bounds on the classifiers' weight vectors or on the key and query vectors, (ii) the classifiers are not required to be Lipschitz, and (iii) the norm bound on the value embeddings can be exponential in the embedding dimension without changing the conclusion of the lower bound.
Our proof focuses on a geometric obstruction that arises from simultaneous $\min$ and $\max$ computation, whereas the proof of \citeauthor{yu2026effectattentionheadcount} uses analytic limitations of neural networks.

\citet{tesfaye2026two} and \citet{viswanathan2026xor} both give separations between one and two attention heads, post-processed by linear classifiers, using problems that are very different from our simultaneous $\min$ and $\max$ problem.
\citet{tesfaye2026two} prove a one-versus-two head separation using a problem that they call Endpoint Selection Problem (ESP).
In ESP, the input is a four-tuple $(u,v,i,\#)$, where $u$ and $v$ come from a finite set $V$, $i$ comes from another set $\set{1,2}$ (disjoint from $V$), and $\#$ is a fixed query token (like \texttt{[CLS]}, per the convention from BERT~\citep{devlin2019bertpretrainingdeepbidirectional}), with $\# \notin V \cup \set{1,2}$.
The correct output is $u$ if $i=1$ and is $v$ if $i=2$.
The challenge is that an attention head is required to provide this output at the position corresponding to the query token $\#$.
\citeauthor{tesfaye2026two} show that a single attention head post-processed by a linear classifier cannot solve ESP, but two attention heads can do so with constant embedding dimension and $\log\card{V}$ precision.
As the authors note, the impossibility result does not hold if the output can be taken from the position of the $i$ token.
So the difficulty arises from a specific requirement of how an attention head is to be used.

\citet{viswanathan2026xor} studies the two-bit XOR problem, where the input is $(x_1,x_2) \in \set{0,1}^2$ and the output is $x_1 \oplus x_2$.
It is shown that for any single attention head, the vector produced on inputs $(0,1)$ and $(1,0)$ is on the line segment joining the vectors produced on inputs $(0,0)$ and $(1,1)$.
Therefore, a linear classifier cannot separate the $(0,1)$ and $(1,0)$ cases from the $(0,0)$ and $(1,1)$ cases.
Notably, this simple geometric argument does not depend on the dimension or the precision level of the embeddings.
\citeauthor{viswanathan2026xor} also shows that using two attention heads readily solves the problem: essentially, one head implements an AND gate and the other head implements an OR gate.
We generalize their observations to $n$-bit XOR, and consider attention layers with multiple heads and post-processing by general polynomial threshold functions.

Some prior works establish the inability of self-attention layers or (variants of) transformers to compute $n$-bit XOR.
\citet{hahn2020theoretical} and \citet{hao2022formallanguagerecognitionhard} study a variant of constant-size transformers that use "unique hard attention" instead of the standard softmax attention, and show that they cannot compute $n$-bit XOR.
\citet{kozachinskiy2026paritysensitivitytransformers} shows that every multi-head attention layer post-processed by a fixed-size ReLU network has average sensitivity $\tilde O\del{\sqrt{n}}$, and hence cannot compute functions of higher average sensitivity such as $n$-bit XOR.
These results do not apply to attention layers that are post-processed by functions of size possibly growing with $n$, or to attention layers with a number of heads that may grow with $n$.

\citet{adler2026capacitybasedrationalemultiheadattention} studies the role of multiple attention heads from the perspective of memorization capacity.
They find that when query/key vectors across all heads in an attention layer share a fixed (dimension) budget, the number of directed relationships between tokens that can be memorized is higher for layers with many heads than for layers with one head (or few heads, empirically).
Using multiple heads reduces the interference from superposition in the embedding space.
Their analysis focuses solely on the expressiveness of attention scores; in particular, they do not study the effect of value vectors.
Our results are complementary: we focus on natural computational tasks supported by attention layers, and our analysis hinges on geometric constraints due to the way value vectors are combined.
Our lower bounds put no restrictions on the query/key vectors whatsoever.

\subsection{Proof techniques}

Our lower bound for the minimum/maximum scenario is based on two key ingredients.
The first ingredient is the Erd\H{o}s--Szekeres theorem~\citep{erdos1935combinatorial,steele1995variations}, which states that every sequence of distinct real numbers of length $T+1$ must either contain a subsequence of length $\geq \sqrt{T}+1$ that is either increasing or decreasing.
In our setting, if the attention weights assigned to some collection of numbers are increasing, then the smallest number in this collection will get relatively low attention weight compared to larger numbers.
Hence, in order for a linear classifier to compute the minimum, the value embeddings associated with the smallest number in this collection should "stand out" among the value embeddings for larger numbers (e.g., by being orthogonal to the other value embeddings).
This line of reasoning leads to a system of constraints for a large number of the value embeddings.
Similarly, if attention weights are decreasing, then we get analogous constraints on the value embeddings in order for a linear classifier to compute the maximum.
The Erd\H{o}s--Szekeres theorem guarantees that, for every choice of attention weights, either the minimum or the maximum task implies constraints on a large number of the value embeddings.

The second ingredient is a translation of the constraints on the value embeddings into lower bounds on the dimensionality or precision level of the value embeddings.
Although the value embeddings are not required to be orthogonal, they must satisfy an "irrepresentable condition" similar to that of \citet{zhao2006modelselectionconsistencylasso}.
We use a volume argument, similar to that of \citet{awerbuch2008online} in their analysis of barycentric spanners, to show that any collection of vectors from an integer lattice satisfying the constraints must either have large dimension or have exponentially large norm.
The lattice and norm constraint is easily translated to a constraint on the bit precision level of the value embeddings.

Our lower bound for $n$-bit XOR is a simple extension of the observations by~\citet{viswanathan2026xor} and \citet{kozachinskiy2026paritysensitivitytransformers}.
Specifically, we use a rational function representation of attention heads over the finite token space $\set{0,1}$---a representation that was also used by \citet{kozachinskiy2026paritysensitivitytransformers}---and then appeal to a known results about the threshold degree of $n$-bit XOR.
The argument works to provide a lower bound for any Boolean function in terms of the threshold degree.

The matching upper bounds for $n$-bit XOR (and, in fact, any symmetric Boolean function) are also based on rational function representations achievable by multi-head attention layers.
Each attention head outputs a vector scaled by the reciprocal of an affine function of the input string's weight (i.e., the number of $1$'s in the input), a standard technique used in previous attention head constructions for parity~\citep[e.g.,][]{chiang2023tighterboundsexpressivitytransformer}.
An elementary argument based on factorization then gives a construction where the dimension of the value embeddings is equal to the degree of the post-processing polynomial threshold function.

To reduce the value embedding dimension, our approach reduces to the following question about sign-representing the parity function:
What is (a bound on) the smallest $d$, such that for all positive integers $D$ and $H$, there are univariate polynomials $P_1,\dotsc,P_d$, all of degree at most $H$, and a $d$-variate polynomial $g$ of degree at most $D$, such that $\sign(g(P_1(t),\dotsc,P_d(t))) = (-1)^t$ for all $t \in \set{0,1,\dotsc,n}$?
(Here, $H$ corresponds to the number of attention heads, and $D$ is the degree of the polynomial threshold function.)
We obtain a bound by constructing small additive bases in the sense of~\citet{mossige1987extremal,challis2010some} and others (a.k.a.~postage stamp bases), but with a restriction on the basis elements.
Specifically, for a given pair of positive integers $(D,H)$, we require an additive $D$-basis $\calB$ with range $DH$, with the additional restriction that all elements of $\calB$ come from $[H]$.
The dimension we achieve this way is equal to the cardinality of such an additive basis (plus one).
To see this, note that, by definition, the additive $D$-basis $\calB$ grants a representation of every $\ell \in [DH]$ as the sum of at most $D$ (not necessarily distinct) elements of $\calB$.
Letting $P_i(t) = t^{\beta_i}$ for each $\beta_i \in \calB$, any monomial $t^\ell$ for $\ell \in [DH]$ is equal to the evaluation of some monomial $z_1^{c_1} z_2^{c_2} \dotsm$ of total degree $c_1 + c_2 + \dotsb \leq D$ at $(z_1,z_2,\dotsc) = (P_1(t),P_2(t),\dotsc)$.
We show the existence of the required additive $D$-bases of size $\leq 2p+1$ whenever $D \gtrsim pH^{1/p}$ for any $p\geq1$.
Combining this with the simpler approach (where the dimension is at most $D$), it follows that dimension $O(\min\set{D,\log H})$ can always be achieved.

\section{Preliminaries}
\label{sec:prelims}

In this \namecref{sec:prelims}, we define the basic notions used in our results and proofs.

\paragraph{Attention.}

We consider the standard softmax attention mechanism with only a single query vector $q$, which can be regarded as corresponding to a special \texttt{[CLS]} token (per the convention from BERT~\citep{devlin2019bertpretrainingdeepbidirectional}) presented alongside the actual input $(x_1,\dotsc,x_n)$.
The \emph{attention operator} $\Att \colon (\R \times \R^d)^n \to \R^d$ is defined as follows:
\begin{equation*}
  \Att (\ip{q, k_j},v_j)_{j=1}^n
  :=
  \frac{
    \sum_{j=1}^n
    \exp(\ip{q, k_j})
    v_j
  }{
    \sum_{j=1}^n
    \exp(\ip{q, k_j})
  }
  .
\end{equation*}
Here, $(k_j,v_j)_{j=1}^n$ are the $n$ pairs of key and value vectors corresponding to the $n$ input elements, and $d$ is the dimension of the value vectors.
Because we are only concerned with a single query vector, which appears only through an inner product with key vectors, the dimension of the query and key vectors is irrelevant, although all of our constructions can be realized with one-dimensional query and key vectors.
(In fact, each $\exp(\ip{q,k_j})$ can be replaced by any positive number $\alpha_j > 0$; we use the standard exponential and inner product form for simplicity and conformity.)

An \emph{attention head} with value vector dimension $d$ for an input token domain $\Omega$ is a mapping from $\Omega^n \to \R^d$ parameterized by $(\ip{q,k_\omega}, v_\omega) \in \R \times \R^d$ for all $\omega \in \Omega$, and computes $\Att (\ip{q,k_{x_j}},v_{x_j})_{j=1}^n$ on every input $(x_1,\dotsc,x_n) \in \Omega^n$.
An \emph{$H$-head attention layer} is a collection of $H$ attention heads (with the same value vector dimension $d$) that computes the sum of outputs of all $H$ attention heads on a given input.
We will use either $\Omega = [M] := \set{1,\dotsc,M}$ for a positive integer $M$, or $\Omega = \set{0,1}$.

\paragraph{Linear classifiers.}

An \emph{unambiguous $M$-class linear classifier} $f \colon \R^d \to [M] \cup \set{\bot}$ is parameterized by $M$ weight vectors $\theta_1,\dotsc,\theta_M \in \R^d$; on input $z \in \R^d$, it returns $\Argmax_{i \in [M]} \ip{\theta_i, z}$, where $\Argmax$ returns $\bot$ if there is a tie for the $\argmax$, and otherwise it returns the unique $\argmax$.
When inputs come from a finite domain, the $\bot$ value is avoided by minimally perturbing the weight vectors.

\paragraph{Polynomials and sign-representations.}

Throughout this paper, only polynomials with real-valued coefficients are considered.
The sole exceptions are $\R^d$-valued polynomials in \Cref{sec:ptf_lb} (so the coefficients are real $d$-vectors), which can be regarded as $d$ separate real-valued polynomials.

An $n$-variate polynomial $p$ \emph{sign-represents} a Boolean function $f \colon \set{0,1}^n \to \set{-1,1}$ if $f(x) = \sign(p(x))$ for all $x \in \set{0,1}^n$.
The \emph{threshold degree} of a Boolean function $f \colon \set{0,1}^n \to \set{-1,1}$ is the minimum degree $D$ such that there exists an $n$-variate polynomial $p$ of degree $D$ that sign-represents $f$.
(The composition of $\sign$ and a polynomial function is a \emph{polynomial threshold function}.)
A Boolean function $f \colon \set{0,1}^n \to \set{-1,1}$ is \emph{symmetric} if it only depends on the weight of the input $\abs{x} := \card{\Set{ i \given i \in [n], x_i = 1 }} = \sum_{i=1}^n x_i$, i.e., there exists $F \colon \set{0,1,\dotsc,n} \to \set{-1,1}$ such that $f(x) = F(\abs{x})$ for all $x \in \set{0,1}^n$.
For example, the $n$-bit XOR function $x \mapsto (-1)^{\abs{x}}$ is symmetric.

\paragraph{Additive bases.}

For positive integers $D$ and $T$, a set of positive integers $\calB$ is an \emph{additive $D$-basis with range $T$} if every non-negative integer at most $T$ can be written as a sum of at most $D$ (not necessarily distinct) elements of $\calB$.
The empty sum is taken to be equal to zero.

\section{Main results}

\subsection{Attention-based minimum and maximum computation}
\label{sec:minmax}

The following \namecref{thm:minmax} is our main result for the first scenario concerning computation of the minimum and maximum of a list of $n$ integers.

\begin{theorem}
  \label{thm:minmax}
  Fix integers $M\geq2$ and $n\geq2$.
  Suppose there are the following:
  \begin{itemize}
    \item precision level $p \in \Z_+$ and norm bound $\ell \geq 2^{-p}$;

    \item query/key values $\ip{q,k_1},\dotsc,\ip{q,k_M} \in \R$;

    \item value vectors $v_1,\dotsc,v_M \in \R^d$
      such that for every $i \in [M]$,
      every component of $v_i$ is an integer multiple of $2^{-p}$, and the Euclidean norm of $v_i$ is at most $\ell$;

    \item unambiguous $M$-class linear classifiers $f^{\min}, f^{\max} \colon \R^d \to [M] \cup \set{\bot}$;

  \end{itemize}
  such that for all $(x_1,\dotsc,x_n) \in [M]^n$,
  \begin{align*}
    f^{\min}(\Att (\ip{q,k_{x_j}},v_{x_j})_{j=1}^n) & = \min\set{x_1,\dotsc,x_n} ; \\
    f^{\max}(\Att (\ip{q,k_{x_j}},v_{x_j})_{j=1}^n) & = \max\set{x_1,\dotsc,x_n} .
  \end{align*}
  Then
  \begin{equation*}
    d \del*{ 1 + p + \log_2(\ell) }
    \geq
    \min\set*{ \floor*{1 + \sqrt{M-1}}, \floor*{\frac{n-1}{2}} }
    .
  \end{equation*}
\end{theorem}

\Cref{thm:minmax} implies that at least one of $d$, $p$, and $\log_2(\ell)$ must grow roughly as the square-root of $\min\set{\sqrt{M},n}$.
If $M \gtrsim n^2$, then we may derive the weaker but simpler conclusion
\begin{equation*}
  \max\set{d, p, \log(\ell) } \gtrsim \sqrt{n} .
\end{equation*}
Observe that $\ell$ can be exponential in $d+p$ without changing this conclusion about $d$ and $p$.

In \Cref{sec:min_only}, we describe an attention head and linear classifier for computing $\min$, with $d = \ceil{8\ln(M)}$, $\ell = \sqrt{\ceil{8\ln(M)}}$, and $p=0$.
A small change makes the construction work for $\max$.
So two such attention heads can support both $\min$ and $\max$.
For any $M = \poly(n)$, the "resource requirement" ($d$, $p$) is exponentially smaller than that of one attention head.

The proof of \Cref{thm:minmax} is given in \Cref{sec:minmax_proof}.

\subsection{Attention-based XOR computation}
\label{sec:xor}

For the second scenario, we generalize the results of \citet{viswanathan2026xor} from two-bit XOR to $n$-bit XOR for all $n$.

\begin{theorem}
  \label{thm:ptf_lb}
  Let $f \colon \set{0,1}^n \to \set{-1,1}$ have threshold degree $T$.
  Suppose there are the following:
  \begin{itemize}
    \item $H$ attention heads, where for each $h \in [H]$, the $h$-th head is specified by:
      \begin{itemize}
        \item query/key values $\ip{q^{(h)},k_0^{(h)}}, \ip{q^{(h)},k_1^{(h)}} \in \R$;

        \item value vectors $v_0^{(h)}, v_1^{(h)} \in \R^d$;
      \end{itemize}

    \item $d$-variate polynomial $g$ of degree $D$;

  \end{itemize}
  such that for all $(x_1,\dotsc,x_n) \in \set{0,1}^n$,
  \begin{equation}
    \label{eq:correct_ptf}
    \sign\del*{g\del*{\sum_{h=1}^H \Att (\ip{q^{(h)},k_{x_j}^{(h)}},v_{x_j}^{(h)})_{j=1}^n}} = f(x_1,\dotsc,x_n) .
  \end{equation}
  Then
  \begin{equation*}
    D \times H \geq T .
  \end{equation*}
\end{theorem}

The $n$-bit XOR (a.k.a.~parity) function has threshold degree $n$~\citep{minsky1969perceptrons}.
Hence, we obtain the following \namecref{cor:xor_lb}.

\begin{corollary}
  \label{cor:xor_lb}
  If $f$ is the $n$-bit XOR function in the setting of \Cref{thm:ptf_lb}, then $D \times H \geq n$.
\end{corollary}

Similar to the result of \citet{viswanathan2026xor} for the two-bit XOR, \Cref{thm:ptf_lb} and \Cref{cor:xor_lb} do not restrict the dimension or bit precision level of the value vectors (or the key and query vectors).

To show the tightness of our lower bound, we give constructions of $H$-head attention layers post-processed by a polynomial threshold function of degree $D$, for all choices of $D$ and $H$ satisfying $D \times H \geq T$, in which the value vectors have dimension at most $O(\min\set{D,\log H})$.

\begin{theorem}
  \label{thm:ptf_ub}
  Let $f \colon \set{0,1}^n \to \set{-1,1}$ be a symmetric Boolean function with threshold degree $T$.
  For any positive integers $D$ and $H$ satisfying $D \times H \geq T$, there exist $H$ attention heads---specified as in \Cref{thm:ptf_lb} with value vector dimension $d \leq\min \set{D, 2\max\set{1,\ceil{\log_2 H}}+2}$---and a polynomial $g \colon \R^d \to \set{-1,1}$ of degree $D$ such that \Cref{eq:correct_ptf} holds for all $(x_1,\dotsc,x_n) \in \set{0,1}^n$.
\end{theorem}

If $D$ is at least a constant positive power of $H$ (e.g., $D \gtrsim H^{0.01}$), or if $D$ itself is a constant, then our proof shows that the value vector dimension can be a constant.
We leave open whether the dimension can be further improved below $O(\min\set{D,\log H})$ for intermediate values of $D$.

The proofs of \Cref{thm:ptf_lb,thm:ptf_ub} are given in \Cref{sec:ptf_proof}.

\section{Proof of \Cref{thm:minmax}}
\label{sec:minmax_proof}

In this \namecref{sec:minmax_proof}, we prove \Cref{thm:minmax}.

\subsection{Histogram representation}

Given an input $(x_1,\dotsc,x_n) \in [M]^n$, consider its histogram vector $h \in \Z_+^M$ such that
\begin{equation*}
  h_i := \sum_{j=1}^n \ind{x_j = i} \qquad \text{for all $i \in [M]$} .
\end{equation*}
Note that $h_1 + \dotsb + h_M = n$.
Let
\begin{equation*}
  \alpha_i := \exp(\ip{q, k_i}) > 0 \qquad \text{for all $i \in [M]$}
\end{equation*}
be the (unnormalized) attention weights.
Then we have
\begin{equation*}
  Z := \sum_{i=1}^M h_i \alpha_i > 0
  \qquad\text{and}\qquad
  \Att (\ip{q,k_{x_j}},v_{x_j})_{j=1}^n
  = \frac1Z \sum_{i=1}^M h_i \alpha_i v_i
  .
\end{equation*}

\subsection{Monotonic subsequence of attention weights}

The following is the first key ingredient for the proof of \Cref{thm:minmax}.

\begin{lemma}
  \label{lem:monotonic_subsequence}
  If $M \geq (N-1)^2+1$, then there exists $(x_1,\dotsc,x_N) \in [M]^N$ with $x_1 < \dotsb < x_N$ such that $(\alpha_{x_j})_{j=1}^N$ is either non-decreasing or non-increasing.
\end{lemma}

\begin{proof}
  This is a consequence of the Erd\H{o}s--Szekeres theorem~\citep{erdos1935combinatorial,steele1995variations}.
\end{proof}

Intuitively, a non-decreasing (non-increasing) $(\alpha_{x_j})_{j=1}^N$ presents an obstacle for $\min$ ($\max$).
This intuition is developed in \Cref{sec:minimum,sec:maximum} next.

\subsection{Constraints implied by a correct minimum classifier}
\label{sec:minimum}

Let $\theta_1,\dotsc,\theta_M$ be the weight vectors for $f^{\min}$.
Throughout this \namecref{sec:minimum}, we assume that 
\begin{equation*}
  f^{\min}(\Att (\ip{q,k_{x_j}},v_{x_j})_{j=1}^n) = \min\set{x_1,\dotsc,x_n}
\end{equation*}
for all $(x_1,\dotsc,x_n) \in [M]^n$.
We show that the correctness of the minimum classifier implies many linear constraints on the value vectors.

For all $x,y,i \in [M]$, define
\begin{equation*}
  \Delta_i^{x,y} := \alpha_i \ip{\theta_x - \theta_y, v_i} .
\end{equation*}
This is the (potential) contribution of $i$ to the classifier's comparison of $x$ and $y$.
Indeed, using the histogram representation for $\Att (\ip{q,k_{x_j}}, v_{x_j})_{j=1}^n$ and linearity, we have
\begin{equation*}
  \ip{\theta_x - \theta_y, \Att (\ip{q,k_{x_j}},v_{x_j})_{j=1}^n}
  =
  \ip*{\theta_x - \theta_y,
    \frac1Z \sum_{i=1}^M h_i \alpha_i v_i
  }
  =
  \frac1Z
  \sum_{i=1}^M h_i \Delta_i^{x,y}
  .
\end{equation*}
Since $Z > 0$, we have
\begin{equation*}
  \ip{\theta_x, \Att (\ip{q,k_{x_j}},v_{x_j})_{j=1}^n}
  >
  \ip{\theta_y, \Att (\ip{q,k_{x_j}},v_{x_j})_{j=1}^n}
  \iff
  \sum_{i=1}^M h_i \Delta_i^{x,y} > 0
  .
\end{equation*}

The next \namecref{lem:min_comparisons} shows that, for the comparison between $x$ and $y$ (for $x < y$), the presence of $x$ in the input has a positive contribution, and presence of larger elements $z \geq y$ can only have smaller (in magnitude) contributions.

\begin{restatable}{lemma}{mincomparisons}
  \label{lem:min_comparisons}
  If $x < y$, then $\Delta_x^{x,y} > 0$, and $\abs{\Delta_z^{x,y}} < \frac1{n-1} \Delta_x^{x,y}$ for all $z \geq y$.
\end{restatable}

\begin{proof}[Proof sketch]
  Fix $x, y \in [M]$ with $x < y$.
  Suppose the input is $(x,\dotsc,x)$, so the minimum is $x$.
  For the output of the linear classifier to be $x$, we must have $\sum_{i=1}^M h_i \Delta_i^{x,y} > 0$, where $(h_i)_{i=1}^M$ is the histogram vector for this input.
  But $\sum_{i=1}^M h_i \Delta_i^{x,y} = n \Delta_x^{x,y}$.
  Therefore $\Delta_x^{x,y} > 0$, proving the first part of the claim.
  The proof of the second part uses similar reasoning as above for some other simple inputs.
  For example, by considering the inputs $(y,\dotsc,y)$ and $(x,y,\dotsc,y)$, we find that $\Delta_y^{x,y} < 0$ and $-\Delta_y^{x,y} < \Delta_x^{x,y} / (n-1)$.
  See \Cref{sec:deferred_proofs} for full details.
\end{proof}

The next \namecref{lem:min_triangular_obstruction} shows that, for a subset of $[M]$ whose attention weights are non-decreasing, the inequalities from \Cref{lem:min_comparisons} can be translated to a system of linear constraints on the corresponding elements' value vectors.
The inequalities are similar to the "irrepresentable condition" of \citet{zhao2006modelselectionconsistencylasso}, with the $\lambda_i$ in \Cref{lem:min_triangular_obstruction} serving as dual certificates for the impossibility of representing $v_{x_i}$ as a linear combination of $v_{x_{i+1}},\dotsc,v_{x_N}$ with small coefficients.

\begin{lemma}
  \label{lem:min_triangular_obstruction}
  Suppose there exists $(x_1,\dotsc,x_N) \in [M]^N$ with $x_1 < \dotsb < x_N$ such that $(\alpha_{x_j})_{j=1}^N$ is non-decreasing.
  For each $i \in \set{1,\dotsc,N-1}$, there exists $\lambda_i \in \R^d$ such that $\ip{\lambda_i, v_{x_i}} = 1$, and $\abs{\ip{\lambda_i, v_{x_j}}} < \frac1{n-1}$ for all $j > i$.
\end{lemma}

\begin{proof}
  Fix $i \in \set{1,\dotsc,N-1}$.
  By \Cref{lem:min_comparisons}, and the fact that $\alpha_x > 0$ for all $x \in [M]$, we have
  \begin{equation}
    \label{eq:Delta_i_positive}
    \ip{\theta_{x_i} - \theta_{x_{i+1}}, v_{x_i}} > 0
  \end{equation}
  and, for all $j > i$,
  \begin{equation}
    \label{eq:abs_ip_j_ub}
    \abs{\ip{\theta_{x_i} - \theta_{x_{i+1}}, v_{x_j}}}
    < \frac{\alpha_{x_i}}{\alpha_{x_j}} \cdot \frac1{n-1} \ip{\theta_{x_i} - \theta_{x_{i+1}}, v_{x_i}}
    \leq \frac1{n-1} \ip{\theta_{x_i} - \theta_{x_{i+1}}, v_{x_i}}
  \end{equation}
  where the second inequality uses \Cref{eq:Delta_i_positive} and the fact that $\alpha_{x_j} \geq \alpha_{x_i}$ for $j > i$.
  So let
  \begin{equation*}
    \lambda_i := \frac1{\ip{\theta_{x_i} - \theta_{x_{i+1}}, v_{x_i}}} \del{\theta_{x_i} - \theta_{x_{i+1}}}
    .
  \end{equation*}
  Then $\ip{\lambda_i, v_{x_i}} = 1$ and $\abs{\ip{\lambda_i, v_{x_j}}} < 1/(n-1)$ for all $j > i$ follow from \Cref{eq:Delta_i_positive} and \Cref{eq:abs_ip_j_ub}.
\end{proof}

\subsection{Constraints implied by a correct maximum classifier}
\label{sec:maximum}

Let $\tilde\theta_1,\dotsc,\tilde\theta_M$ be the weight vectors for $f^{\max}$.
Throughout this \namecref{sec:maximum}, we assume that 
\begin{equation*}
  f^{\max}(\Att (\ip{q,k_{x_j}},v_{x_j})_{j=1}^n) = \max\set{x_1,\dotsc,x_n}
\end{equation*}
for all $(x_1,\dotsc,x_n) \in [M]^n$.
For all $x,y,i \in [M]$, define $\tilde\Delta_i^{x,y} := \alpha_i \ip{\tilde\theta_x - \tilde\theta_y, v_i}$, so we have
\begin{equation*}
  \ip{\tilde\theta_x, \Att (\ip{q,k_{x_j}},v_{x_j})_{j=1}^n}
  >
  \ip{\tilde\theta_y, \Att (\ip{q,k_{x_j}},v_{x_j})_{j=1}^n}
  \iff
  \sum_{i=1}^M h_i \tilde\Delta_i^{x,y} > 0
  .
\end{equation*}

\begin{lemma}
  \label{lem:max_comparisons}
  If $x > y$, then $\tilde\Delta_x^{x,y} > 0$, and $\abs{\tilde\Delta_z^{x,y}} < \frac1{n-1} \tilde\Delta_x^{x,y}$ for all $z \leq y$.
\end{lemma}

\begin{lemma}
  \label{lem:max_triangular_obstruction}
  Suppose there exists $(x_1,\dotsc,x_N) \in [M]^N$ with $x_1 < \dotsb < x_N$ such that $(\alpha_{x_j})_{j=1}^N$ is non-increasing.
  For each $i \in \set{2,\dotsc,N}$, there exists $\mu_i \in \R^d$ such that $\ip{\mu_i, v_{x_i}} = 1$ and $\abs{\ip{\mu_i, v_{x_j}}} < \frac1{n-1}$ for all $j < i$.
\end{lemma}

The proof of \Cref{lem:max_comparisons} is completely analogous to that of \Cref{lem:min_comparisons}.
The proof of \Cref{lem:max_triangular_obstruction} uses \Cref{lem:max_comparisons} in a way analogous to how the proof of \Cref{lem:min_triangular_obstruction} uses \Cref{lem:min_comparisons}.

\subsection{Approximately triangular configurations}

The second key ingredient for the proof of \Cref{thm:minmax} is the geometric obstruction captured in \Cref{lem:triangular} below.
It is used in conjunction with either \Cref{lem:min_triangular_obstruction} or \Cref{lem:max_triangular_obstruction} in the proof of the main theorem; the vectors $u_1,\dotsc,u_N$ will be (scalings of) an appropriately chosen list of value vectors in our application.

\begin{lemma}
  \label{lem:triangular}
  Fix $N \geq 1$ and $L \geq 1$.
  Suppose non-zero vectors $u_1,\dotsc,u_N, w_2,\dotsc,w_N \in \R^d$ satisfy, for some $\epsilon \in \intoo{0,1/R}$:
  \begin{align*}
    \ip{w_i,u_i} & = 1 \quad \text{for all $i > 1$} , \\
    \abs{\ip{w_i,u_j}} & \leq \epsilon \quad \text{for all $j < i$} ,
  \end{align*}
  where $R\geq1$ is the dimension of $\Span\del{\set{u_1,\dotsc,u_N}}$.
  Furthermore, suppose each $u_i$ has integer components and Euclidean norm at most $L$.
  Then
  \begin{equation*}
    N \leq R\del*{ 1 + \frac{\log(L)}{\log\del*{\frac1{R\epsilon}}} } .
  \end{equation*}
  More precisely, if $r_i$ is the dimension of $\Span\del{\set{u_1,\dotsc,u_i}}$, and $N_r = \card{\Set{ i > 1 \given r_i = r_{i-1} = r }}$, then
  \begin{equation}
    \label{eq:main}
    \sum_{r=1}^R N_r \log\del*{\frac1{r\epsilon}} \leq R \log(L)
    .
  \end{equation}
\end{lemma}

To get some intuition for \Cref{lem:triangular}, consider the limit $\epsilon\to0$, so we have a perfectly upper triangular configuration: there is a matrix $W$ such that $W^\T [u_1 \mid \dotsb \mid u_N]$ is an $N \times N$ upper triangular matrix with ones on the diagonal.
This means that this matrix product is non-singular, so the rank of each matrix in the product is $N$.
Hence, $N = R$ is required.

When we allow $\epsilon > 0$, the vectors $u_1,\dotsc,u_N$ need not be linearly independent.
However, the approximate triangular configuration implies that if $u_i \in \Span(\set{u_1,\dotsc,u_{i-1}})$, and we write $u_i$ as a linear combination of $r$ of these preceding vectors, where $r$ is the dimension of the span, then at least one of the coefficients must have large magnitude.
This implies a particular "volume growth" as we find vectors $u_i$ that remain in the span of preceding vectors, similar to the analysis of barycentric spanners due to \citet{awerbuch2008online}.
The number of such volume growth steps is then limited by assumptions of bounded precision and bounded norm, which we leverage through \Cref{claim:volumeub,claim:volumelb} given below and proved in \Cref{sec:deferred_proofs}.

\begin{restatable}{claim}{volumeub}
  \label{claim:volumeub}
  If every column of a matrix $B$ has Euclidean norm at most $L$, then
  $\det(B^\T B) \leq L^{2r}$, where $r$ is the number of columns of $B$.
\end{restatable}

\begin{restatable}{claim}{volumelb}
  \label{claim:volumelb}
  If every entry of a matrix $A$ is an integer,
  and $A$ has full column rank, then $\det(A^\T A) \geq 1$.
\end{restatable}

We now prove \Cref{lem:triangular}.

\begin{proof}[Proof of \Cref{lem:triangular}]
  We analyze a process that considers the (non-zero) vectors $u_1,\dotsc,u_N$ in order and produces two sequences of matrices $(A_i)_{i=1}^N$ and $(B_i)_{i=1}^N$.
  For each $i \in [N]$, let $U_i := \set{u_1,\dotsc,u_i}$, $S_i := \Span(U_i)$, and $r_i := \dim(S_i)$, so $1 = r_1 \leq \dotsb \leq r_N = R$.
  The process is shown \Cref{fig:basis_process}.

  \begin{figure}[t]
    \fbox{%
      \parbox{\linewidth}{%
        \begin{itemize}[topsep=0pt,leftmargin=0pt]
          \item[] Let $A_1 := \sbr{u_1}$ and $B_1 := \sbr{u_1}$.
          \item[] For each $i = 2,\dotsc,N$:
            \begin{itemize}
              \item If $u_i \notin S_{i-1}$ (so $S_{i-1} \neq S_i = \Span(S_{i-1} \cup \set{u_i})$, $r_i = r_{i-1} + 1$), then: \hfill (\emph{rank-increasing step})
                \begin{itemize}
                  \item Let $A_i := \sbr{ A_{i-1} \mid u_i }$ and $B_i := \sbr{ B_{i-1} \mid u_i }$.

                \end{itemize}
              \item Else (so $u_i \in S_{i-1} = S_i$, $r := r_i = r_{i-1}$): \hfill (\emph{rank-preserving step})
                \begin{itemize}
                  \item Let $A_i := A_{i-1}$.

                  \item Let $b_1,\dotsc,b_r$ denote the columns of $B_{i-1}$.

                  \item Let $(t_1,\dotsc,t_r) \in \R^r$ be coefficients such that $u_i = t_1 b_1 + \dotsb + t_r b_r$.

                  \item Pick any $k \in \argmax_{j \in [r]} \abs{t_j}$.

                  \item Let $B_i := \sbr{ b_1 \mid \dotsb \mid b_{k-1} \mid u_i \mid b_{k+1} \mid \dotsb \mid b_r }$.

                \end{itemize}

            \end{itemize}

        \end{itemize}%
      }%
    }
    \caption{The process analyzed in the proof of \Cref{lem:triangular}.}
    \label{fig:basis_process}
  \end{figure}

  The definition of $A_i$ in for-loop step $i$ ensures (i) every column in $A_i$ is a vector from $U_i$, and (ii) the columns of $A_i$ form an ordered basis for $S_i$; the same is true for $B_i$.
  Because the columns of $A_i$ and the columns of $B_i$ form bases for the same subspace, there is a unique (and invertible) $r_i \times r_i$ change-of-basis matrix $C_i$ such that $B_i = A_i C_i$.
  The same is true for $i=1$, with $C_1 := \sbr{1}$.
  The determinant of a product of square matrices is the product of their determinants, so
  \begin{equation}
    \label{eq:change_of_basis_det}
    \det(B_i^\T B_i)
    = \det(A_i^\T A_i) \det(C_i)^2
    .
  \end{equation}
  With the $i=N$ case of \Cref{eq:change_of_basis_det} and the fact that $\det(C_1) = 1$, we obtain the telescoping identity:
  \begin{equation}
    \label{eq:det_telescoping}
    \det(B_N^\T B_N)
    = \det(A_N^\T A_N) \prod_{i=2}^N \del*{ \frac{\det(C_i)}{\det(C_{i-1})} }^2 .
  \end{equation}

  We next analyze the ratios $\det(C_i)/\det(C_{i-1})$.
  Suppose step $i$ is a rank-increasing step.
  Then
  \begin{equation*}
    A_i
    = \sbr{ A_{i-1} \mid u_i }
    \qquad \text{and} \qquad
    B_i
    = \sbr{ B_{i-1} \mid u_i }
    = \sbr{ A_{i-1} C_{i-1} \mid u_i }
    =
    \underbrace{
      \sbr{ A_{i-1} \mid u_i }
    }_{A_i}
    \begin{bmatrix}
      C_{i-1} & \\
              & 1
    \end{bmatrix}
    .
  \end{equation*}
  Therefore
  \begin{equation*}
    C_i
    =
    \begin{bmatrix}
      C_{i-1} & \\
              & 1
    \end{bmatrix}
    ,
  \end{equation*}
  which implies
  \begin{equation}
    \label{eq:det_ratio_constant}
    \del*{ \frac{\det(C_i)}{\det(C_{i-1})} }^2 = 1 .
  \end{equation}

  Now suppose instead that step $i$ is a rank-preserving step.
  Adopt the notations $r$, $b_1,\dotsc,b_r$, $k$, and $t := (t_1,\dotsc,t_r)$ from \Cref{fig:basis_process}.
  Since $B_i$ swaps out $b_k$ for $u_i = t_1 b_1 + \dotsb + t_r b_r$ in $B_{i-1} = \sbr{b_1 \mid \dotsb \mid b_r}$, we can write $B_i = B_{i-1} T$, where
  \begin{equation*}
    T := 
    \left[
      \begin{array}{ccc}
        I_{k-1} & \vert   & \\
                & t          & \\
                & \vert & I_{r-k}
      \end{array}
    \right]
  \end{equation*}
  is the $r \times r$ matrix obtained by replacing the $k$-th column of the $r \times r$ identity matrix by $t$ (as a column vector).
  Since we also have $A_i = A_{i-1}$,
  \begin{equation*}
    B_i
    = B_{i-1} T
    = (A_{i-1} C_{i-1}) T
    = A_i (C_{i-1} T) ,
  \end{equation*}
  which implies $C_i = C_{i-1} T$, and hence $\det(C_i) = \det(C_{i-1}) \det(T)$.
  The Laplace expansion of $\det(T)$ along the $k$-th column of $T$ implies that $\det(T) = t_k$.
  Therefore
  \begin{equation}
    \label{eq:det_ratio}
    \del*{ \frac{\det(C_i)}{\det(C_{i-1})} }^2 = t_k^2 .
  \end{equation}
  Now we show a lower bound on $\abs{t_k}$.
  By linearity, the triangle inequality, and the choice of $k$,
  \begin{align}
    \abs{\ip{w_i,u_i}}
    & = \abs{t_1 \ip{w_i,b_1} + \dotsb + t_r \ip{w_i,b_r}}
    \nonumber \\
    & \leq \abs{t_1} \cdot \abs{\ip{w_i,b_1}} + \dotsb + \abs{t_r} \cdot \abs{\ip{w_i,b_r}}
    \nonumber \\
    & \leq \abs{t_k} \cdot \del*{ \abs{\ip{w_i,b_1}} + \dotsb + \abs{\ip{w_i,b_r}} }
    .
    \label{eq:abs_linear_bound}
  \end{align}
  Using the assumptions on $w_i$ and $u_1,\dotsc,u_i$, and the fact that every column of $B_{i-1}$ comes from $U_{i-1}$, it follows that
  \begin{equation}
    \label{eq:linear_bound}
    \ip{w_i,u_i} = 1
    \qquad\text{and}\qquad
    \abs{\ip{w_i,b_1}} + \dotsb + \abs{\ip{w_i,b_r}} \leq r\epsilon .
  \end{equation}
  Therefore, combining \Cref{eq:abs_linear_bound,eq:linear_bound} gives $\abs{t_k} \geq 1/(r\epsilon)$, and hence by \Cref{eq:det_ratio},
  \begin{equation}
    \label{eq:det_ratio_increase}
    \del*{ \frac{\det(C_i)}{\det(C_{i-1})} }^2 \geq \del*{ \frac1{r\epsilon} }^2 .
  \end{equation}

  Now we return to the telescoping identity \Cref{eq:det_telescoping}.
  For each $r \in [R]$, let $N_r$ denote the number of rank-preserving steps $i$ with $r_i = r_{i-1} = r$.
  Then by \Cref{eq:det_ratio_constant,eq:det_ratio_increase},
  \begin{equation*}
    \prod_{i=2}^N \del*{ \frac{\det(C_i)}{\det(C_{i-1})} }^2
    \geq \prod_{r=1}^R \del*{ \frac1{r\epsilon} }^{2N_r} .
  \end{equation*}
  Moreover, by \Cref{claim:volumeub} and \Cref{claim:volumelb},
  \begin{equation*}
    \det(B_N^\T B_N) \leq L^{2R}
    \qquad\text{and}\qquad
    \det(A_N^\T A_N) \geq 1
    .
  \end{equation*}
  Combining the inequalities in these last two displays with \Cref{eq:det_telescoping} gives
  \begin{equation*}
    L^{2R}
    \geq
    \det(B_N^\T B_N)
    = \det(A_N^\T A_N) \prod_{i=2}^N \del*{ \frac{\det(C_i)}{\det(C_{i-1})} }^2
    \geq \prod_{r=1}^R \del*{ \frac1{r\epsilon} }^{2N_r}
    .
  \end{equation*}
  Since $\log\del{\tfrac1{r\epsilon}} \geq \log\del{\tfrac1{R\epsilon}} > 0$ for all $r \in [R]$ by the assumption $\epsilon \in \intoo{0,1/R}$, taking logarithms and simplifying gives \Cref{eq:main}.
  We conclude that the total number of for-loop steps, $N-1$, satisfies
  \begin{align*}
    N-1
    & =
    \text{(number of rank-increasing steps)}
    + \text{(number of rank-preserving steps)}
    \\
    & = \del*{ R-1 } + \del*{ N_1 + \dotsb + N_R }
    \\
    & \leq R-1
    + \frac1{\log\del*{\frac1{R\epsilon}}} \sum_{r=1}^R N_r \log\del*{\frac1{r\epsilon}}
    \\
    & \leq R-1 + \frac{R\log(L)}{\log\del*{\frac1{R\epsilon}}}
    \qquad \text{(by \Cref{eq:main})} .
    \qedhere
  \end{align*}
\end{proof}

\subsection{Finishing the proof}

We now finish the proof of \Cref{thm:minmax}.

\begin{proof}[Proof of \Cref{thm:minmax}]
  We may assume that $v_i \neq 0$ for all $i \in [M]$.
  This is because if the input has $x_j = i$ for all $j \in [n]$, then $\Att (\ip{q,k_{x_j}},v_{x_j})_{j=1}^n = v_i$, which cannot be zero if $f^{\min}$ and $f^{\max}$ are to output non-$\bot$ values.

  We apply \Cref{lem:monotonic_subsequence} to obtain $(x_1,\dotsc,x_N) \in [M]^N$ for
  \begin{equation*}
    N := \min\set*{ \floor*{1 + \sqrt{M-1}}, \floor*{\frac{n-1}{2}} }
  \end{equation*}
  with $x_1 < \dotsb < x_N$ and $(\alpha_{x_j})_{j=1}^N$ either non-decreasing or non-increasing.
  If $(\alpha_{x_j})_{j=1}^N$ is non-decreasing, then we apply \Cref{lem:min_triangular_obstruction}, and set $u_i := 2^pv_{x_{N-i+1}}$ and (for $i \geq 2$) $w_i := 2^{-p}\lambda_{N-i+1}$.
  If $(\alpha_{x_j})_{j=1}^N$ is non-increasing, then we apply \Cref{lem:max_triangular_obstruction}, and set $u_i := 2^pv_{x_i}$ and (for $i \geq 2$) $w_i := 2^{-p}\mu_i$.
  In either case, we obtain non-zero vectors $u_1,\dotsc,u_N, w_2,\dotsc,w_N$ satisfying the preconditions of \Cref{lem:triangular} with $L := 2^p\ell$ and $\epsilon := 1/(n-1)$.
  The span of $u_1,\dotsc,u_N$ has dimension
  \begin{equation}
    \label{eq:rank_bound}
    R \leq \min\set*{d, \floor*{1 + \sqrt{M-1}}, \floor*{\frac{n-1}{2}} } ,
  \end{equation}
  and hence
  \begin{equation}
    \label{eq:growth_factor_bound}
    \log\frac{1}{R\epsilon} \geq \log(2)
    .
  \end{equation}
  Combining \Cref{eq:rank_bound,eq:growth_factor_bound} with \Cref{lem:triangular}, we have
  \begin{equation*}
    N
    \leq
    R \del*{ 1 + \frac{\log(L)}{\log\del*{\frac1{R\epsilon}}} }
    \leq
    d \del*{ 1 + \log_2(L) }
    =
    d \del*{ 1 + p + \log_2(\ell) }
    .
    \qedhere
  \end{equation*}
\end{proof}

\section{Proofs of \Cref{thm:ptf_lb,thm:ptf_ub}}
\label{sec:ptf_proof}

In this \namecref{sec:ptf_proof}, we prove \Cref{thm:ptf_lb,thm:ptf_ub}.

\subsection{Notation}

Throughout this \namecref{sec:ptf_proof}, we use the notations
\begin{equation}
  \label{eq:att_param}
  \ip{q,k_0}, \ip{q,k_1} \in \R \quad \text{and} \quad v_0, v_1 \in \R^d
\end{equation}
to denote the parameters for a single attention head, which on input $(x_1,\dotsc,x_n) \in \set{0,1}^n$, computes
\begin{equation*}
  \Att (\ip{q,k_{x_j}},v_{x_j})_{j=1}^n .
\end{equation*}
And we use the notations
\begin{equation}
  \label{eq:layer_param}
  \ip{q^{(h)},k_0^{(h)}}, \ip{q^{(h)},k_1^{(h)}} \in \R \quad \text{and} \quad v_0^{(h)}, v_1^{(h)} \in \R^d , \qquad \forall h \in [H]
\end{equation}
to denote the parameters for an $H$-head attention layer, which on input $(x_1,\dotsc,x_n) \in \set{0,1}^n$, computes
\begin{equation*}
  \sum_{h=1}^H \Att (\ip{q^{(h)},k_{x_j}^{(h)}},v_{x_j}^{(h)})_{j=1}^n .
\end{equation*}

\subsection{Rational representations of attention layers}
\label{sec:ptf_lb}

We first show \Cref{thm:ptf_lb}.
The proof begins with the standard observation that the output of an $H$-head attention layer on any input $x \in \set{0,1}^n$ is given by a rational function $P(x) / Q(x)$, where $P$ is an $\R^d$-valued polynomial of degree at most $H$, and $Q$ is a polynomial of degree at most $H$ that is positive on $\set{0,1}^n$.

\begin{lemma}
  \label{lem:rational_representation}
  Consider any $H$ attention heads with parameters \Cref{eq:layer_param} (and value vector dimension $d$).
  There exist an $\R^d$-valued $n$-variate polynomial $P$ of degree at most $H$, and an $n$-variate polynomial $Q$ of degree at most $H$ such that $Q(x_1,\dotsc,x_n) > 0$ and $\sum_{h=1}^H \Att (\ip{q^{(h)}, k_{x_j}^{(h)}},v_{x_j}^{(h)})_{j=1}^n = P(x_1,\dotsc,x_n) / Q(x_1,\dotsc,x_n)$ for all $(x_1,\dotsc,x_n) \in \set{0,1}^n$.
\end{lemma}

\begin{proof}
  Consider a single attention head with parameters \Cref{eq:att_param}.
  For any input $(x_1,\dotsc,x_n) \in \set{0,1}^n$,
  \begin{equation*}
    \Att (\ip{q,k_{x_j}},v_{x_j})_{j=1}^n
    = \frac{A(x_1,\dotsc,x_n)}{Z(x_1,\dotsc,x_n)} ,
  \end{equation*}
  where we define $A \colon \R^n \to \R^d$ and $Z \colon \R^n \to \R$ by
  \begin{align*}
    A(x_1,\dotsc,x_n)
    & := \sum_{i=1}^n (1 - x_i) \exp(\ip{q, k_0}) v_0 + x_i \exp(\ip{q, k_1}) v_1
    \intertext{and}
    Z(x_1,\dotsc,x_n)
    & := \sum_{i=1}^n (1 - x_i) \exp(\ip{q, k_0}) + x_i \exp(\ip{q, k_1})
    .
  \end{align*}
  Observe that $A$ and $Z$ are affine functions of $(x_1,\dotsc,x_n)$, with $Z$ taking positive values on $\set{0,1}^n$.

  Therefore, the sum of $H$ attention heads can be written as
  \begin{equation*}
    \sum_{h=1}^H \Att (\ip{q^{(h)},k_{x_j}^{(h)}},v_{x_j}^{(h)})_{j=1}^n
    =
    \sum_{h=1}^H \frac{A^{(h)}(x_1,\dotsc,x_n)}{Z^{(h)}(x_1,\dotsc,x_n)}
  \end{equation*}
  where $A^{(1)},\dotsc,A^{(H)} \colon \R^n \to \R^d$
  and $Z^{(1)},\dotsc,Z^{(H)} \colon \R^n \to \R$
  are all affine functions, with each of $Z^{(1)},\dotsc,Z^{(H)}$ taking positive values on $\set{0,1}^n$.
  Define $Z^{(1:H)} := \prod_{h=1}^H Z^{(h)}$, which takes positive values on $\set{0,1}^n$; and for each $h \in [H]$ , define $Z^{(-h)} := \prod_{h'\neq h} Z^{(h')}$.
  Then we have
  \begin{equation}
    \label{eq:rational}
    \sum_{h=1}^H \frac{A^{(h)}(x_1,\dotsc,x_n)}{Z^{(h)}(x_1,\dotsc,x_n)}
    = 
    \frac{\sum_{h=1}^H A^{(h)}(x_1,\dotsc,x_n) Z^{(-h)}(x_1,\dotsc,x_n)}{Z^{(1:H)}(x_1,\dotsc,x_n)} 
    .
  \end{equation}
  For each $h \in [H]$, $Z^{(-h)}$ is the product of $H-1$ affine functions and hence is a polynomial of degree at most $H-1$.
  Therefore, each of the numerator and denominator in the right-hand side expression in \Cref{eq:rational} is a polynomial of degree at most $H$.
\end{proof}

We next consider the effect of composing a polynomial function with a rational function whose denominator is positive over a domain.

\begin{lemma}
  \label{lem:poly_rational_composition}
  Suppose
  $g$ is a $d$-variate polynomial of degree at most $D$,
  $P$ is an $n$-variate $\R^d$-valued polynomial of degree at most $H$, and
  $Q$ is an $n$-variate polynomial of degree at most $H$ with $Q(x) > 0$ for all $x \in \set{0,1}^n$.
  Then there exists an $n$-variate polynomial $R$ of degree at most $DH$ such that, for all $x \in \set{0,1}^n$, $\sign(g(P(x)/Q(x))) = \sign(R(x))$.
\end{lemma}

\begin{proof}
  For each $j \in [d]$, write $P_j$ for the $j$-th component of $P$.
  For a multi-index $K = (K_1,\dotsc,K_d) \in \N_0^d := \Set{ (z_1,\dotsc,z_d) \given z_i \in \Z, z_i \geq 0 \, \forall i \in [d] }$, write $\abs{K} := \sum_{j=1}^d K_j$.
  Since $g$ has degree at most $D$, we can write
  \begin{equation*}
    g(z_1,\dotsc,z_d) = \sum_{\substack{K \in \N_0^d : \\ \abs{K} \leq D}} c_K \prod_{j=1}^d z_j^{K_j}
  \end{equation*}
  for some real coefficients $c_K$.
  Formally define $R(x) := Q(x)^D g(P(x)/Q(x))$, which after expanding the expression for $g$ becomes
  \begin{equation*}
    R(x)
    = Q(x)^D
    \sum_{\substack{K \in \N_0^d : \\ \abs{K} \leq D}} c_K \prod_{j=1}^d (P_j(x)/Q(x))^{K_j}
    =
    \sum_{\substack{K \in \N_0^d : \\ \abs{K} \leq D}} c_K
    Q(x)^{D - \abs{K}}
    \prod_{j=1}^d P_j(x)^{K_j} 
    .
  \end{equation*}
  The term in the final sum corresponding to multi-index $K$ has degree
  \begin{equation*}
    \deg\del*{ Q(x)^{D-\abs{K}} \prod_{j=1}^d P_j(x)^{K_j} }
    \leq (D - \abs{K}) \deg(Q) + \sum_{j=1}^d K_j \deg(P_j)
    \leq (D - \abs{K}) H + \abs{K} H = DH .
  \end{equation*}
  Therefore $R$ is an $n$-variate polynomial with $\deg(R) \leq DH$.
  Moreover, since $Q(x)^D > 0$ for all $x \in \set{0,1}^n$, we have
  \begin{equation*}
    \sign(g(P(x)/Q(x)))
    = \sign(Q(x)^D g(P(x)/Q(x)))
    = \sign(R(x))
    \quad \forall x \in \set{0,1}^n .
    \qedhere
  \end{equation*}
\end{proof}

Combining \Cref{lem:rational_representation,lem:poly_rational_composition} shows that if the composition of a degree-$D$ polynomial $g$ and an $H$-head attention layer sign-represents $f$ on $\set{0,1}^n$, then there is a degree-$DH$ polynomial $R$ that sign-represents $f$ as well.
This implies that $D \times H \geq T$, proving \Cref{thm:ptf_lb}.

\subsection{Realizing a polynomial threshold function with attention heads}

The remainder of \Cref{sec:ptf_proof} is dedicated to proving \Cref{thm:ptf_ub}.
In \Cref{lem:att_realization}, we show a basic computation achievable by a single attention head, which will serve as the basis for richer computations achievable by multiple attention heads.

\begin{lemma}
  \label{lem:att_realization}
  For any vectors $a, b \in \R^d$ and scalar $c > -1$, there is an attention head with parameters \Cref{eq:att_param} and value vector dimension $d$, such that for all $x = (x_1,\dotsc,x_n) \in \set{0,1}^n$,
  \begin{equation*}
    \Att (\ip{q,k_{x_j}},v_{x_j})_{j=1}^n
    = a + \frac1{n + c \abs{x}} b .
  \end{equation*}
\end{lemma}

\begin{proof}
  Set $\ip{q,k_0} := 0$, $\ip{q,k_1} := \ln(1+c)$, $v_0 := a + \frac1nb$, $v_1 := a + \frac{1}{n(1+c)}b$.
  Then
  \begin{align*}
    \sum_{j=1}^n
    \del*{ (1-x_j) \exp(\ip{q,k_0}) (v_0 - a) + x_j \exp(\ip{q,k_1}) (v_1 - a) }
    & = \sum_{j=1}^n
    \del*{ \frac{(1-x_j) b}{n} + \frac{x_j (1+c) b}{n(1+c)} }
    = b
    \intertext{and}
    \sum_{j=1}^n
    \del*{ (1-x_j) \exp(\ip{q,k_0}) + x_j \exp(\ip{q,k_1}) }
    & = \sum_{j=1}^n \del*{ (1-x_j) + x_j (1+c) }
    = n + c \abs{x} .
  \end{align*}
  Therefore
  \begin{equation*}
    \Att (\ip{q,k_{x_j}},v_{x_j})_{j=1}^n
    = a + \Att (\ip{q,k_{x_j}},v_{x_j} - a)_{j=1}^n
    = a + \frac{1}{n+c\abs{x}} b .
    \qedhere
  \end{equation*}
\end{proof}

In \Cref{lem:poly_basis}, we give a standard construction of a basis for the space of bounded degree polynomials using products of distinct affine functions.
We defer the proof to \Cref{sec:deferred_proofs}.

\begin{restatable}{lemma}{polybasis}
  \label{lem:poly_basis}
  Fix any non-zero scalar $n$ and distinct non-zero scalars $c_1, \dotsc, c_H$, and define the univariate polynomials
  \begin{equation}
    \label{eq:poly_basis}
    Q_0(t) := \prod_{h=1}^H (n + c_ht) ,
    \qquad
    Q_h(t) := \prod_{h' \neq h} (n + c_{h'}t) \quad \forall h \in [H] .
  \end{equation}
  Then $Q_0,Q_1,\dotsc,Q_H$ form a basis for the vector space of polynomials of degree at most $H$.
\end{restatable}

\Cref{lem:attention_layer_realization} is the core of our construction: we show that $H$ attention heads with value vector dimension $d$ can realize the any $d$ polynomials in $\abs{x}$ of degree at most $H$, up to a positive scaling that ultimately does not affect the sign after being post-processed by a homogeneous polynomial.

\begin{lemma}
  \label{lem:attention_layer_realization}
  Suppose $P_1,\dotsc,P_d$ are univariate polynomials, each of degree at most $H$.
  Then there is an $H$-head attention layer with parameters \Cref{eq:layer_param} and value vector dimension $d$, such that for all $x = (x_1,\dotsc,x_n) \in \set{0,1}^n$,
  \begin{equation*}
    \sum_{h=1}^H \Att (\ip{(q^{(h)},k_{x_j}^{(h)}},v_{x_j}^{(h)})_{j=1}^n
    = \frac{(P_1(\abs{x}),\dotsc,P_d(\abs{x}))}{Q_0(\abs{x})} ,
  \end{equation*}
  where $Q_0$ is a real-valued function that is positive on the non-negative reals.
  Furthermore, if $g$ is a homogeneous $d$-variate polynomial, then for all $x = (x_1,\dotsc,x_n) \in \set{0,1}^n$,
  \begin{equation*}
    \sign\del*{
      g\del*{
        \sum_{h=1}^H \Att (\ip{(q^{(h)},k_{x_j}^{(h)}},v_{x_j}^{(h)})_{j=1}^n
      }
    } =
    \sign\del*{ g\del*{P_1(\abs{x}),\dotsc,P_d(\abs{x})} }
    .
  \end{equation*}
\end{lemma}

\begin{proof}
  Fix any distinct positive scalars $c_h>0$ for all $h \in [H]$, and let $Q_0, Q_1, \dotsc, Q_H$ be the corresponding polynomials defined in \Cref{eq:poly_basis}, which form an (ordered) basis for the vector space of polynomials of degree at most $H$ as per \Cref{lem:poly_basis}.
  Then for each $i \in [d]$, there exists scalars $\alpha_i, \beta_i^{(1)}, \dotsc, \beta_i^{(H)}$ such that
  \begin{equation*}
    P_i = \alpha_i Q_0  + \sum_{h=1}^H \beta_i^{(h)} Q_h .
  \end{equation*}
  Formally dividing through by $Q_0$ gives
  \begin{equation*}
    \frac{P_i(t)}{Q_0(t)}
    = \alpha_i + \sum_{h=1}^H \frac1{n+c_ht} \beta_i^{(h)}
    = \sum_{h=1}^H \del*{ \frac1H \alpha_i + \frac1{n+c_ht} \beta_i^{(h)} }
    \quad \forall i \in [d] .
  \end{equation*}
  Observe that $Q_0(t) > 0$ for all $t \geq 0$ since $c_h \geq 0$ for all $h \in [H]$.

  Let $\alpha := (\alpha_1,\dotsc,\alpha_d) \in \R^d$ and $\beta^{(h)} := (\beta_1^{(h)},\dotsc,\beta_d^{(h)}) \in \R^d$ for each $h \in [H]$.
  For each $h \in [H]$, we apply \Cref{lem:att_realization} with $a := \alpha / H$, $b := \beta^{(h)}$, $c := c_h > 0$; this provides a construction for $H$ attention heads with parameters \Cref{eq:layer_param} and value vector dimension $d$, such that for all $x = (x_1,\dotsc,x_n) \in \set{0,1}^n$,
  \begin{equation*}
    \Att (\ip{(q^{(h)},k_{x_j}^{(h)}},v_{x_j}^{(h)})_{j=1}^n
    = \frac1H \alpha + \frac1{n + c_h \abs{x}} \beta^{(h)} ,
    \quad \forall h \in [H] .
  \end{equation*}
  Therefore
  \begin{equation*}
    \sum_{h=1}^H \Att (\ip{q^{(h)},k_{x_j}^{(h)}},v_{x_j}^{(h)})_{j=1}^n
    = \alpha + \sum_{h=1}^H \frac1{n + c_h \abs{x}} \beta^{(h)}
    = \frac{(P_1(\abs{x}),\dotsc,P_d(\abs{x}))}{Q_0(\abs{x})} .
  \end{equation*}
  Note that $Q_0(\abs{x})$ is positive because $\abs{x} \geq 0$ for all $x \in \set{0,1}^n$.

  For the final claim (after "Furthermore"), let $D$ denote the degree of $g$, and observe that for any $x = (x_1,\dotsc,x_n) \in \set{0,1}^n$,
  \begin{equation*}
    g\del*{
      \sum_{h=1}^H \Att (\ip{q^{(h)},k_{x_j}^{(h)}},v_{x_j}^{(h)})_{j=1}^n
    }
    =
    g\del*{
      \frac{P_1(\abs{x})}{Q_0(\abs{x})},\dotsc,\frac{P_d(\abs{x})}{Q_0(\abs{x})}
    }
    =
    \frac{g\del{P_1(\abs{x}),\dotsc,P_d(\abs{x})}}{Q_0(\abs{x})^D} ,
  \end{equation*}
  where the last equality follows by homogeneity of $g$.
  Since $Q_0(\abs{x}) > 0$, we have
  \begin{align*}
    \sign\del*{
      g\del*{
        \sum_{h=1}^H \Att (\ip{q^{(h)},k_{x_j}^{(h)}},v_{x_j}^{(h)})_{j=1}^n
      }
    }
    & = \sign\del*{
      \frac{g\del{P_1(\abs{x}),\dotsc,P_d(\abs{x})}}{Q_0(\abs{x})^D}
    }
    \\
    & = \sign\del*{g\del{P_1(\abs{x}),\dotsc,P_d(\abs{x})}}
    .
    \qedhere
  \end{align*}
\end{proof}

\subsection{Degree-restricted compositional sign-representations}

To use \Cref{lem:attention_layer_realization}, we need sign-representations for symmetric Boolean functions that are the form $g(P_1(\abs{x}),\dotsc,P_d(\abs{x}))$ for some polynomials $P_1,\dotsc,P_d$ of degree at most $H$ and a homogeneous $d$-variate polynomial of degree at most $D$.
\Cref{lem:partial_product} shows how to construct such representations with value vector dimension $d = D$ using partial products of the polynomial factorization.

\begin{lemma}
  \label{lem:partial_product}
  Let $F(t)$ be a univariate polynomial of degree $T$.
  Suppose $D$ and $H$ are positive integers with $D \times H \geq T$.
  There exist
  \begin{itemize}
    \item a positive integer $d \leq D$;

    \item a homogeneous $d$-variate polynomial $g$ with degree at most $D$;

    \item $d$ univariate polynomials $P_1(t),\dotsc,P_d(t)$, each of degree at most $H$;

  \end{itemize}
  such that
  \begin{equation*}
    \sign(F(t)) = \sign(g(P_1(t),\dotsc,P_d(t))) .
  \end{equation*}
\end{lemma}

\begin{proof}
  Since $F(t)$ is a (real) univariate polynomial of degree $T$, it has $T$ complex roots $r_1,\dotsc,r_T$ (counting multiplicity), and thus it has a factorization
  \begin{equation*}
    F(t) = c \prod_{j=1}^T (t - r_j) ,
  \end{equation*}
  where $c$ is a real scalar.
  Let $T' \in \set{0,\dotsc,T}$ denote the number of real roots.
  Order the roots so that $r_1,\dotsc,r_{T'}$ are real, and let $Q(t) := \prod_{j = T'+1}^T (t - r_j)$ denote the part of the product corresponding to non-real roots, so $F(t) = c \prod_{j=1}^{T'} (t - r_j) Q(t)$.
  We claim that $Q(t) > 0$.
  To see this, note that any complex root $r = a + b\iu$ with $b \neq 0$ can be paired with another root equal to its complex conjugate $\bar{r} = a - b\iu$, and the product of their contributing factors is
  \begin{equation*}
    (t - r)(t - \bar{r})
    = (t - (a + b\iu))(t - (a - b\iu))
    = (t-a)^2 + b^2 > 0 .
  \end{equation*}

  Since $D \times H \geq T \geq T'$, we may partition the indices $[T']$ into $d \leq D$ disjoint sets $J_1,\dotsc,J_d$ with at most $H$ indices per set.
  For each $i \in [d]$, define
  \begin{equation*}
    P_i(t) := \prod_{j \in J_i} (t - r_j) ,
  \end{equation*}
  which is a polynomial of degree $\card{J_i} \leq H$.
  Also define the $d$-variate polynomial
  \begin{equation*}
    g(z_1,\dotsc,z_d) := \sign(c) z_1 \dotsm z_d ,
  \end{equation*}
  which is homogeneous and has degree at most $D$.
  We have
  \begin{equation*}
    \abs{c} g(P_1(t),\dotsc,P_d(t)) Q(t)
    = c \del*{ \prod_{j \in J_1} (t - r_j) } \dotsm \del*{ \prod_{j \in J_d} (t - r_j) } Q(t)
    = c \prod_{j=1}^{T'} (t - r_j) Q(t)
    = F(t) .
  \end{equation*}
  Therefore, since $Q(t) > 0$,
  \begin{equation*}
    \sign\del*{g(P_1(t),\dotsc,P_d(t))}
    = \sign\del*{\abs{c} g(P_1(t),\dotsc,P_d(t))Q(t)}
    = \sign\del*{F(t)} .
    \qedhere
  \end{equation*}
\end{proof}

When $D$ is large (relative to some positive power of $H$), we can achieve a smaller value vector dimension $d$ using a different approach based on additive bases.
We first show how to carry out this approach when provided a suitable additive $D$-basis in~\Cref{lem:additive_basis_representation}; the conditions requires to guarantee the existence of such an additive $D$-basis are deferred to \Cref{lem:additive_bases}.

\begin{lemma}
  \label{lem:additive_basis_representation}
  Let $F(t)$ be a univariate polynomial of degree $T$.
  Suppose $\calB \subseteq [H]$ is an additive $D$-basis with range $T$.
  Set $d := \card{\calB} + 1$.
  There exist
  \begin{itemize}
    \item a homogeneous $d$-variate polynomial $g$ with degree $D$;

    \item $d$ univariate monomials $P_1(t),\dotsc,P_d(t)$, each of degree at most $H$;

  \end{itemize}
  such that
  \begin{equation*}
    F(t) = g(P_1(t),\dotsc,P_d(t)) .
  \end{equation*}
\end{lemma}

\begin{proof}
  Denote the $d-1$ elements of $\calB$ by $\beta_1, \dotsc, \beta_{d-1}$.
  Define $P_i(t) := t^{\beta_i}$ for each $i \in [d-1]$, and $P_d(t) := 1$.
  Since $\calB \subseteq [H]$, each $P_i(t)$ is a monomial of degree at most $H$.
  Since $\calB$ is an additive $D$-basis with range $T$, each $\ell \in \set{0,1,\dotsc,T}$ can be written as
  \begin{equation*}
    \ell = c_{\ell,1} \beta_1 + \dotsb + c_{\ell,d-1} \beta_{d-1}
  \end{equation*}
  for some non-negative integers $c_{\ell,1},\dotsc,c_{\ell,d-1}$ that sum to at most $D$.
  Furthermore, for each $\ell$, define $c_{\ell,d} := D - \sum_{i=1}^{d-1} c_{\ell,i} \geq 0$, so we have
  \begin{equation}
    \label{eq:homogeneous}
    \sum_{i=1}^d c_{\ell,i} = D , \quad \forall \ell \in \set{0,1,\dotsc,T} .
  \end{equation}
  Expand $F(t)$ in the monomial basis:
  \begin{equation*}
    F(t) = \sum_{\ell=0}^T a_\ell t^\ell ;
  \end{equation*}
  correspondingly, define the $d$-variate polynomial $g$ by
  \begin{equation*}
    g(z_1,\dotsc,z_d)
    := \sum_{\ell=0}^T a_\ell \prod_{i=1}^d z_i^{c_{\ell,i}} .
  \end{equation*}
  Note that \Cref{eq:homogeneous} ensures that $g$ is homogeneous and has degree $D$.
  Furthermore,
  \begin{align*}
    g(P_1(t),\dotsc,P_d(t))
    & = \sum_{\ell=0}^T a_\ell 
    P_1(t)^{c_{\ell,1}}
    \dotsm
    P_{d-1}(t)^{c_{\ell,d-1}}
    P_d(t)^{c_{\ell,d}}
    \\
    & =
    \sum_{\ell=0}^T a_\ell
    \del*{ t^{\beta_1} }^{c_{\ell,1}}
    \dotsm
    \del*{ t^{\beta_{d-1}} }^{c_{\ell,d-1}}
    \del*{ 1 }^{c_{\ell,d}}
    \\
    & = \sum_{\ell=0}^T a_\ell t^\ell = F(t) .
    \qedhere
  \end{align*}
\end{proof}

The existence of small additive $D$-bases from $[H]$ is provided in the next \namecref{lem:additive_bases} under an assumption about the relationship between $D$ and $H$.
The proof is given in \Cref{sec:additive_bases}.

\begin{restatable}{lemma}{additivebases}
  \label{lem:additive_bases}
  Let $D$, $H$, $p$ be positive integers with
  $D+2 \geq 2p(b-1)$
  where $b := \ceil{H^{1/p}}$.
  Every non-negative integer at most $DH$ is equal to the sum of at most $D$ copies of elements from
  \begin{equation*}
    \calB := \set{ H } \cup \calB_1 \cup \calB_2
  \end{equation*}
  where $\calB_1 := \set{ b^0, b^1, \dotsc, b^{p-1} } \cap \set{1,\dotsc,H-1}$ and $\calB_2 := \Set{ H - \beta \given \beta \in \calB_1 }$.
  Hence, there exists an additive $D$-basis $\calB \subseteq [H]$ with range $DH$ of cardinality $\card{\calB} \leq 2p+1$.
\end{restatable}

\subsection{Finishing the proof}

We can now complete the proof of \Cref{thm:ptf_ub}.

\begin{proof}[Proof of \Cref{thm:ptf_ub}]
  By Minsky-Papert symmetrization~\citep{minsky1969perceptrons}, it is known that the threshold degree of a symmetric Boolean function $f \colon \set{0,1}^n \to \set{-1,1}$ is also the minimum degree $T$ of a univariate polynomial $F(t)$ such that $x \mapsto F(\abs{x})$ sign-represents $f$.
  So fix a univariate polynomial $F(t)$ of degree $T$, and assume $D \times H \geq T$.
  Let $p := \max\set{1,\ceil{\log_2(H)}}$.
  There are two cases to consider.
  \begin{itemize}
    \item
      Case 1: $D \leq 2p-3$.
      We use \Cref{lem:partial_product} to obtain a sign-representation of $F(t)$ by $g(P_1(t),\dotsc,P_d(t))$ with $d = D$, which we use with \Cref{lem:attention_layer_realization} to get the required $H$ attention heads and $d$-variate polynomial $g$.
      Note that this construction is applicable for all values of $D$, not just those with $D \leq 2p-3$.

    \item
      Case 2: $D \geq 2p-2$.
      Let $b := \ceil{H^{1/p}} = 2$, so $D+2 \geq 2p(b-1)$.
      We can therefore obtain an additive $D$-basis $\calB \subseteq [H]$ with range $T \leq DH$ from \Cref{lem:additive_bases}.
      We use $\calB$ and \Cref{lem:additive_basis_representation} to obtain a representation of $F(t) = g(P_1(t),\dotsc,P_d(t))$ with $d = \card{\calB} + 1 \leq 2p+2$, which again we use with \Cref{lem:attention_layer_realization} to get the required $H$ attention heads and $d$-variate polynomial $g$.
      \qedhere

  \end{itemize}
\end{proof}

To obtain an explicit construction of the multi-head self-attention layer and post-processing polynomial threshold function, we require access to a univariate polynomial $F(t)$ such that $x \mapsto F(\abs{x})$ sign-represents the target symmetric Boolean function.
When \Cref{lem:partial_product} is used, it suffices to have the real roots of $F(t)$ and the sign of its leading coefficient.
When \Cref{lem:additive_basis_representation} is used instead, we need the coefficients of $F(t)$ in the monomial basis.
All other aspects of the proof of \Cref{thm:ptf_ub} are constructive.

\section{Discussion}

The most patent limitation of \Cref{thm:minmax} is that it only considers post-processing by linear classifiers.
Although linear predictability is a natural benchmark for assessing multi-task representations~\citep[e.g.,][]{du2021fewshotlearninglearningrepresentation,tripuraneni2022provablemetalearninglinearrepresentations}, it is also common to use non-linear predictors such as kernel machines and neural networks.
It is not clear to us how to extend the lower bound in \Cref{thm:minmax} to hold against non-linear classifiers, nor is it clear how to get around the lower bound with such classifiers (without also greatly increasing the embedding dimension or precision level).
If further restrictions are put on the classifier (e.g., smoothness), then techniques similar to those from \citet{yu2026effectattentionheadcount} can be used to establish lower bounds against some families of non-linear classifiers like neural networks (see \Cref{sec:lipschitz}).

\Cref{thm:ptf_lb,thm:ptf_ub} together show that increasing the number of heads in an attention layer (when post-processed by a polynomial threshold function) serves the role of amplifying the threshold degree of (symmetric) Boolean functions that can be computed.
A large value vector dimension cannot compensate for an deficient number of heads.
Our $H$-head attention layer constructions use dimension $O(\min\set{D,\log H})$.
It would be interesting to understand if this dimension is necessary, or what fine-grained role the dimension plays for sign-representing Boolean functions.

Our results concern standard softmax attention without any bells and whistles like positional embeddings and layer normalization.
Considering the effect of these features, as well as the joint role of multiple heads and multiple layers, are interesting directions for future research.

\bibliographystyle{plainnat}
\bibliography{refs}

\clearpage

\appendix
\crefalias{section}{appendix}

\section{Attention-based minimum computation}
\label{sec:min_only}

In this \namecref{sec:min_only}, we describe an attention head and linear classifier that computes the minimum of a given list of $n$ integers from $[M]$.

\begin{proposition}
  \label{prop:minmax}
  Fix integers $M\geq2$ and $n\geq1$, and set $d := \ceil*{8\ln(M)}$.
  There exists query/key values $\ip{q,k_1},\dotsc,\ip{q,k_M} \in \R$, value vectors $v_1,\dotsc,v_M \in \set{-1,1}^d$, and an unambiguous $M$-class linear classifier $f^{\min} \colon \R^d \to [M] \cup \set{\bot}$ such that for all $(x_1,\dotsc,x_n) \in [M]^n$,
  \begin{equation*}
    f^{\min}(\Att (\ip{q,k_{x_j}},v_{x_j})_{j=1}^n) = \min\set{x_1,\dotsc,x_n} .
  \end{equation*}
\end{proposition}

In this construction, the value vectors have dimension $d = \ceil{8\ln(M)}$ and Euclidean norm exactly $\ell := \sqrt{d} = \sqrt{\ceil{8\ln(M)}}$, and their components are integer multiples of $2^{-p}$ for $p := 0$.

The same construction works for computing the maximum, just negating all query/key values in the attention head construction.

\begin{proof}[Proof of \Cref{prop:minmax}]
  Choose vectors $v_x \in \set{-1,1}^d$ for each $x \in [M]$ such that every pair of vectors $v_x, v_y$ for $x \neq y$ differ in at least $d/4$ positions.
  Such a choice of vectors is guaranteed to be possible by the Gilbert-Varshamov bound since $d \geq 8\ln(M)$.
  Define $\ip{q,k_x} := -\eta x$ for each $x \in [M]$, where $\eta := \ceil{\ln(n(nd-1))}$.
  Let the weight vectors for $f^{\min}$ be $\theta_x := v_x$ for all $x \in [M]$.

  Consider an input $(x_1,\dotsc,x_n)$ where $x = \min\set{x_1,\dotsc,x_n}$.
  Then a short calculation shows that the output $z$ of the attention head can be written as
  \begin{equation*}
    z = p v_x + (1-p) \bar{v} ,
  \end{equation*}
  where $\bar{v}$ is a convex combination of other $v_y$'s with $y \neq x$, and $p \in \intcc{1 - 1/(nd),1}$.
  Then
  \begin{equation*}
    \ip{\theta_x,z}
    = \ip{v_x,z}
    = p \ip{v_x,v_x} + (1-p) \ip{v_x,\bar{v}}
    \geq p d - (1-p) d
    = d - \frac{2}{n} .
  \end{equation*}
  For $y \neq x$, since $v_y$ and $v_x$ differ in at least $d/4$ positions,
  \begin{equation*}
    \ip{\theta_y,z}
    = \ip{v_y,z}
    = p \ip{v_y,v_x} + (1-p) \ip{v_y,\bar{v}}
    \leq p \frac{d}{2} + (1-p) d
    \leq \frac{d}{2} + \frac{1}{2n} .
  \end{equation*}
  So we have $\ip{\theta_x,z} > \ip{\theta_y,z}$ for all $y \neq x$ (since $d > 5/n$).
\end{proof}

\section{Deferred proofs}
\label{sec:deferred_proofs}

\mincomparisons*

\begin{proof}
  Fix $x, y \in [M]$ with $x < y$.
  We consider different inputs $(x_1,\dotsc,x_n) \in [M]^n$ corresponding to different histogram vectors $h$, and derive consequences of correct $\min$ computation on these inputs.
  \begin{enumerate}
    \item 
      Suppose $h_x = n$ and $h_i = 0$ for all $i \neq x$.
      The minimum is $x$, so
      \begin{equation*}
        \sum_{i=1}^M h_i \Delta_i^{x,y} > 0 .
      \end{equation*}
      Moreover,
      \begin{equation*}
        \sum_{i=1}^M h_i \Delta_i^{x,y}
        = h_x \Delta_x^{x,y}
        = n \Delta_x^{x,y} .
      \end{equation*}
      Hence
      \begin{equation*}
        \Delta_x^{x,y} > 0
        .
      \end{equation*}
      This proves the first part of the claim.

    \item 
      Suppose $h_y = n$ and $h_i = 0$ for all $i \neq y$.
      The minimum is $y$, so
      \begin{equation*}
        \sum_{i=1}^M h_i \Delta_i^{x,y} < 0 .
      \end{equation*}
      Moreover,
      \begin{equation*}
        \sum_{i=1}^M h_i \Delta_i^{x,y}
        = h_y \Delta_y^{x,y}
        = n \Delta_y^{x,y} .
      \end{equation*}
      Hence
      \begin{equation}
        \label{eq:Delta_y_negative}
        \Delta_y^{x,y} < 0
        .
      \end{equation}

    \item 
      Suppose $h_x = 1$ and $h_z = n-1$ for some $z \geq y$, and $h_i = 0$ for all $i \notin \set{x,z}$.
      The minimum is $x$, so
      \begin{equation*}
        \sum_{i=1}^M h_i \Delta_i^{x,y} > 0 .
      \end{equation*}
      Moreover,
      \begin{equation*}
        \sum_{i=1}^M h_i \Delta_i^{x,y}
        = h_x \Delta_x^{x,y} + h_z \Delta_z^{x,y}
        = \Delta_x^{x,y} + (n-1) \Delta_z^{x,y} .
      \end{equation*}
      Hence
      \begin{equation}
        \label{eq:Delta_z_lb}
        \Delta_z^{x,y} > -\frac1{n-1} \Delta_x^{x,y} ,
      \end{equation}
      and, in particular, for $z = y$,
      \begin{equation}
        \label{eq:abs_Delta_y_ub}
        \abs{\Delta_y^{x,y}} = -\Delta_y^{x,y} < \frac1{n-1} \Delta_x^{x,y}
      \end{equation}
      where we have used the fact that $\Delta_y^{x,y}$ is negative, per \Cref{eq:Delta_y_negative}.

    \item 
      Suppose $h_y = 1$ and $h_z = n-1$ for some $z > y$, and $h_i = 0$ for all $i \notin \set{y,z}$.
      The minimum is $y$, so
      \begin{equation*}
        \sum_{i=1}^M h_i \Delta_i^{x,y} < 0 .
      \end{equation*}
      Moreover,
      \begin{equation*}
        \sum_{i=1}^M h_i \Delta_i^{x,y}
        = h_y \Delta_y^{x,y} + h_z \Delta_z^{x,y}
        = \Delta_y^{x,y} + (n-1) \Delta_z^{x,y} .
      \end{equation*}
      Hence
      \begin{equation*}
        \Delta_z^{x,y} < -\frac1{n-1} \Delta_y^{x,y} = \frac1{n-1} \abs{\Delta_y^{x,y}} .
      \end{equation*}
      Combining with \Cref{eq:abs_Delta_y_ub}, we have
      \begin{equation}
        \label{eq:Delta_z_ub}
        \Delta_z^{x,y}
        < \frac1{\del{n-1}^2} \Delta_x^{x,y}
        \leq \frac1{n-1} \Delta_x^{x,y}
        .
      \end{equation}
      Combining \Cref{eq:Delta_z_ub} and \Cref{eq:Delta_z_lb} (both of which hold for all $z > y$) gives
      \begin{equation}
        \label{eq:abs_Delta_z_ub}
        \abs{\Delta_z^{x,y}}
        < \frac1{n-1} \Delta_x^{x,y}
        .
      \end{equation}
      Combining \Cref{eq:abs_Delta_y_ub} and \Cref{eq:abs_Delta_z_ub} proves the second part of the claim.
      \qedhere

  \end{enumerate}
\end{proof}

\volumeub*

\begin{proof}
  The claim follows from Hadamard's inequality.
\end{proof}

\volumelb*

\begin{proof}
  Since every entry of $A$ is an integer, so is every entry of $A^\T A$.
  Since $A$ has full column rank, $A^\T A$ is positive definite, and hence $\det(A^\T A) > 0$.
  But $\det(A^\T A)$ is a polynomial in the entries of $A^\T A$ with integer coefficients and hence must evaluate to an integer.
  So $\det(A^\T A) \geq 1$.
\end{proof}

\polybasis*

\begin{proof}
  Each of $Q_0,Q_1,\dotsc,Q_H$ is polynomial of degree at most $H$.
  Hence it suffices to show that they are linearly independent.
  Consider any real scalars $\alpha_0,\alpha_1,\dotsc,\alpha_H$ and define
  \begin{equation*}
    P := \sum_{i=0}^H \alpha_i Q_i .
  \end{equation*}
  Fix $h \in [H]$, and consider the evaluation of $P$ at $t := -n/c_h$.
  Since $Q_i(-n/c_h) = 0$ for all $i \neq h$, we have
  \begin{equation*}
    P(-n/c_h) = \alpha_h \prod_{i \neq h} \del*{ n\del*{ 1 - \frac{c_i}{c_h} } } .
  \end{equation*}
  Since the scalars $c_1,\dotsc,c_H$ are non-zero and distinct, if $P$ is the zero polynomial, then $\alpha_h = 0$ for all $h \in [H]$.
  Since $Q_0$ is evidently not the zero polynomial, it follows that if $\alpha_0 Q_0 \equiv 0$, then $\alpha_0 = 0$ as well.
  So, if $P$ is the zero polynomial, then $\alpha_0 = \alpha_1 = \dotsb = \alpha_H = 0$, which means that $Q_0, Q_1, \dotsc, Q_H$ are linearly independent.
\end{proof}

\section{Additive bases}
\label{sec:additive_bases}

In this \namecref{sec:additive_bases}, we give an approach for constructing additive $D$-bases with range $DH$ using elements from $[H]$.

To give the main idea of the construction, we consider the special case where $H \geq 2$ is a power-of-two (so in \Cref{lem:additive_bases}, we have $p = \log_2 H$ and $b = 2$).
The candidate additive $D$-basis is $\calB = \set{H} \cup \calB_1 \cup \calB_2$, where $\calB_1 = \set{ 2^0, 2^1, 2^2, \dotsc, H/2 }$, and $\calB_2 = \set{ H - 2^0, H - 2^1, H - 2^2, \dotsc, H/2 }$.
Let us test this candidate by trying to find a representation of an arbitrary $n \in [DH]$ as a sum of at most $D$ (not necessarily distinct) elements of $\calB$.
We could try to use $q = \floor{n/H}$ copies of $H$, along with the elements of $\calB_1$ needed to form the remainder $r = n - qH \leq H-1$.
The only reason this representation might not work is that it uses more than $D$ elements, meaning that $q + \log_2 H \geq D+1$.
In this case, we can try to reduce the number of copies of $H$ needed---say, from $q$ to $q - \Delta$---and make up for the difference by using elements of $\calB_2$ to represent $\Delta H + r$.
This latter representation is possible for some non-negative integer $\Delta$ at most $\log_2 H - 1$.
So, in the worst case, we need $q \geq \log_2 H - 1$ to be implied by the condition $q + \log_2 H  \geq D + 1$, which is true for large enough $D$ (i.e., $D \geq 2(\log_2 H  - 1)$ for this special case).

\additivebases*

\begin{proof}
  The definition of $b$ ensures that $b^p \geq H$, and hence every non-negative integer $r$ less than $H$ can be written as
  \begin{equation*}
    r = c_0 b^0 + c_1 b^1 + \dotsb + c_{p-1} b^{p-1}
  \end{equation*}
  where the coefficients $c_0,c_1,\dotsc,c_{p-1}$ are non-negative integers at most $b-1$.
  For any $i \in \set{0,\dotsc,p-1}$, if $c_i > 0$, then $b^i \leq r < H$, and hence $b_i \in \calB_1$.
  Moreover,
  \begin{equation*}
    c_0 + c_1 + \dotsb + c_{p-1} \leq p(b-1) .
  \end{equation*}
  Therefore, $r$ is equal to the sum of at most $p(b-1)$ copies of elements from $\calB_1$.

  Now consider any non-negative integer $n \leq DH$.
  If $n = DH$, then we are done, since $n$ is the sum of $D$ copies of $H$.
  So assume $n \leq DH - 1$, and write
  \begin{equation*}
    n = qH + r
  \end{equation*}
  where $q$ is the integer quotient, and $r$ is the integer remainder.
  Since $n \leq DH - 1$, the quotient is a non-negative integer satisfying
  \begin{equation*}
    q
    = \floor*{ \frac{n}{H} }
    \leq \floor*{ \frac{DH - 1}{H} }
    = \floor*{ D - \frac1H }
    = D - 1
    .
  \end{equation*}
  If $r = 0$, then we are done, since $n$ is the sum of $q \leq D-1$ copies of $H$.
  So henceforth we assume both $q \leq D-1$ and $r \geq 1$.

  We now consider two cases: $q + p(b-1) \leq D$ and $q + p(b-1) \geq D+1$.
  \begin{itemize}
    \item Case 1: $q + p(b-1) \leq D$.
      Since $r \geq 1$ is integer remainder of $n$ after taking out multiples of $H$, it is a positive integer less than $H$.
      So $r$ is equal to the sum of $m$ copies of elements from $\calB_1$ for some $m \leq p(b-1)$.
      Since $n = qH + r$, it follows that $n$ is equal to the sum of $q$ copies of $H$ and $m$ copies of elements from $\calB_1$.
      Since $q + m \leq q + p(b-1) \leq D$, the claim follows in this case.

    \item Case 2: $q + p(b-1) \geq D+1$.
      Let $R := H - r$, which is a positive integer less than $H$.
      So $R$ is equal to the sum of $m$ copies of elements from $\calB_1$ for some $m \leq p(b-1)$:
      \begin{equation*}
        R = \beta_1 + \dotsb + \beta_m , \quad \beta_1, \dotsc, \beta_m \in \calB_1 .
      \end{equation*}
      Then
      \begin{align*}
        n
        & = qH + r \\
        & = (q+1)H - R \\
        & = (q+1-m)H + \sum_{i=1}^m (H - \beta_i) .
      \end{align*}
      Since $q \geq D+1-p(b-1)$, $m \leq p(b-1)$, and (by assumption) $D+2 \geq 2p(b-1)$, it follows that
      \begin{equation*}
        a := q+1-m \geq D+2-2p(b-1)
      \end{equation*}
      is a non-negative integer.
      So $n$ is equal to the sum of $a$ copies of $H$ and $m$ copies of elements from $\calB_2$.
      Since $q \leq D-1$, it follows that $a + m = q+1 \leq D$.
      Hence the claim follows in this case as well.
      \qedhere

  \end{itemize}

\end{proof}

Note that we cannot hope to have such constant-size additive $D$-bases with range $DH$ for arbitrary $(D,H)$.
To see this, consider any set of positive integers $\calB$ of cardinality $k \geq 1$.
Then there are at most $\binom{D+k}{k}$ choices of $(c_1,\dotsc,c_k)$ with non-negative integers $c_1,\dotsc,c_k$ that sum to at most $D$.
But there are $DH+1$ non-negative integers at most $DH$.
If $\calB$ is an additive $D$-basis with range $DH$ and cardinality $k$, then
\begin{equation*}
  \binom{D+k}{k} \geq DH+1 ,
\end{equation*}
which implies $H \leq (D+1)^{k-1}$.

\section{Lower bounds against post-processing by Lipschitz predictors}
\label{sec:lipschitz}

In this \namecref{sec:lipschitz}, we consider attention heads that are post-processed by a Lipschitz predictor.
\Cref{prop:lipschitz} shows that such attention heads cannot support both $\min$ and $\max$ computation unless the Lipschitz constant of the post-processing classifiers or the distance between some pair of value vectors roughly grows at least as the square-root of the input size.
The proof is similar in spirit to that of \citet{yu2026effectattentionheadcount}.
A lower bound on the Lipschitz constant of two-layer neural net with Lipschitz activation functions and bounded parameter matrices implies a lower bound on the number of hidden units in the neural net~\citep[see, e.g.,][Lemma 4]{yu2026effectattentionheadcount}.

Let $\norm{\cdot}$ denote a norm on $\R^d$, and let $\norm{\cdot}_\infty$ denote the infinity-norm on $\R^M$.
Let $\Delta([M]) := \Set{ (p_1,\dotsc,p_M) \given p_i \geq 0 \, \forall i \in [M], p_1 + \dotsb + p_M = 1 }$ denote the space of probability vectors in $\R^M$.
Finally, let $e_1,\dotsc,e_M$ denote the coordinate basis vectors in $\R^M$.

\begin{proposition}
  \label{prop:lipschitz}
  Fix integers $M\geq2$ and $n\geq2$.
  Suppose there are the following:
  \begin{itemize}
    \item error bound $\epsilon \in \intoo{0,1/2}$;

    \item diameter $D\geq0$ and Lipschitz constant $L\geq0$;

    \item query/key values $\ip{q,k_1},\dotsc,\ip{q,k_M} \in \R$;

    \item value vectors $v_1,\dotsc,v_M \in \R^d$
      such that $\norm{v_i - v_j} \leq D$ for all $i, j \in [M]$;

    \item functions $g^{\min}, g^{\max} \colon \R^d \to \Delta([M])$ that are $L$-Lipschitz maps from $(\Conv(\set{v_1,\dotsc,v_M}),\norm{\cdot})$ to $(\Delta([M]), \norm{\cdot}_\infty)$;

  \end{itemize}
  such that for all $(x_1,\dotsc,x_n) \in [M]^n$,
  \begin{align*}
    \norm{g^{\min}(\Att (\ip{q,k_{x_i}},v_{x_i})_{i=1}^n) - e_{\min\set{x_1,\dotsc,x_n}}}_\infty & \leq \epsilon , \\
    \norm{g^{\max}(\Att (\ip{q,k_{x_i}},v_{x_i})_{i=1}^n) - e_{\max\set{x_1,\dotsc,x_n}}}_\infty & \leq \epsilon .
  \end{align*}
  Then
  \begin{equation*}
    DL \geq (1 - 2\epsilon) n .
  \end{equation*}
\end{proposition}

The proof of \Cref{prop:lipschitz} relies on the following \namecref{lem:close_outputs}.

\begin{lemma}
  \label{lem:close_outputs}
  Suppose $M \geq 2$ and $\norm{v_i - v_j} \leq D$ for all $i, j \in [M]$.
  There exist $(x_1,\dotsc,x_n) \in [M]^n$ and $(x_1',\dotsc,x_n') \in [M]^n$ such that
  \begin{equation*}
    \norm{
      \Att (\ip{q,k_{x_j}},v_{x_j})_{j=1}^n - \Att (\ip{q,k_{x_j'}},v_{x_j'})_{j=1}^n
    } \leq \frac{D}{n} ,
  \end{equation*}
  and at least one of the following inequalities hold:
  \begin{align*}
    \min\set{x_1,\dotsc,x_n} & \neq \min\set{x_1',\dotsc,x_n'} , \\
    \max\set{x_1,\dotsc,x_n} & \neq \max\set{x_1',\dotsc,x_n'} .
  \end{align*}
\end{lemma}

\begin{proof}
  Pick any $i, j \in [M]$ with $i < j$.
  Suppose $\alpha_i \leq \alpha_j$, where
  $\alpha_i := \exp\del{ \ip{q,k_i} }$ and
  $\alpha_j := \exp\del{ \ip{q,k_j} }$.
  Set
  \begin{align*}
    (x_1,\dotsc,x_n) & := (i,j,\dotsc,j) , \\
    (x_1',\dotsc,x_n') & := (j,j,\dotsc,j) , \\
    z & := \Att (\ip{q,k_{x_j}},v_{x_j})_{j=1}^n , \\
    z' & := \Att (\ip{q,k_{x_j'}},v_{x_j'})_{j=1}^n .
  \end{align*}
  So we have
  \begin{equation*}
    \min\set{x_1,\dotsc,x_n} \neq \min\set{x_1',\dotsc,x_n'} .
  \end{equation*}
  Moreover,
  \begin{align*}
    z - z'
    & = \frac{\alpha_i v_i + (n-1) \alpha_j v_j}{\alpha_i + (n-1) \alpha_j} - v_j \\
    & = \frac{\alpha_i}{\alpha_i + (n-1) \alpha_j} (v_i - v_j) \\
    & = \frac{1}{1 + (n-1) \alpha_j/\alpha_i} (v_i - v_j) .
  \end{align*}
  Since $\alpha_j / \alpha_i \geq 1$ and $\norm{v_i - v_j} \leq D$, we have
  \begin{equation*}
    \norm{z - z'}
    = \frac{1}{1 + (n-1) \alpha_j/\alpha_i} \norm{v_i - v_j}
    \leq \frac{1}{n} \norm{v_i - v_j} \leq \frac{D}{n} .
  \end{equation*}

  Now instead suppose $\alpha_i > \alpha_j$.
  Set
  \begin{align*}
    (x_1,\dotsc,x_n) & := (i,\dotsc,i,j) , \\
    (x_1',\dotsc,x_n') & := (i,\dotsc,i,i) , \\
    z & := \Att (\ip{q,k_{x_j}},v_{x_j})_{j=1}^n , \\
    z' & := \Att (\ip{q,k_{x_j'}},v_{x_j'})_{j=1}^n .
  \end{align*}
  So we have
  \begin{equation*}
    \max\set{x_1,\dotsc,x_n} \neq \max\set{x_1',\dotsc,x_n'} .
  \end{equation*}
  Moreover,
  \begin{align*}
    z - z'
    & = \frac{(n-1) \alpha_i v_i + \alpha_j v_j}{(n-1) \alpha_i + \alpha_j} - v_i \\
    & = \frac{\alpha_j}{(n-1) \alpha_i + \alpha_j} (v_j - v_i) \\
    & = \frac{1}{(n-1) \alpha_i/\alpha_j + 1} (v_j - v_i) .
  \end{align*}
  Since $\alpha_i / \alpha_j > 1$ and $\norm{v_j - v_i} \leq D$, we have
  \begin{equation*}
    \norm{z - z'}
    = \frac{1}{(n-1) \alpha_i/\alpha_j + 1} \norm{v_j - v_i}
    \leq \frac{1}{n} \norm{v_j - v_i} \leq \frac{D}{n} .
    \qedhere
  \end{equation*}
\end{proof}

\begin{proof}[Proof of \Cref{prop:lipschitz}]
  Fix $(x_1,\dotsc,x_n), (x_1',\dotsc,x_n') \in [M]^n$ with the properties guaranteed in \Cref{lem:close_outputs}, and let
  \begin{align*}
    z & := \Att (\ip{q,k_{x_j}},v_{x_j})_{j=1}^n , \\
    z' & := \Att (\ip{q,k_{x_j'}},v_{x_j'})_{j=1}^n .
  \end{align*}
  Suppose
  \begin{equation*}
    \min\set{x_1,\dotsc,x_n} \neq \min\set{x_1',\dotsc,x_n'} .
  \end{equation*}
  Then by the triangle inequality and the Lipschitz property of $g^{\min}$,
  \begin{align*}
    1 =
    \norm{
      e_{\min\set{x_1,\dotsc,x_n}}
      - e_{\min\set{x_1',\dotsc,x_n'}}
    }_\infty
    & \leq
    \norm{ e_{\min\set{x_1,\dotsc,x_n}} - g^{\min}(z) }_\infty
    \\
    & \qquad
    + \norm{ g^{\min}(z) - g^{\min}(z') }_\infty
    \\
    & \qquad
    + \norm{ g^{\min}(z') - e_{\min\set{x_1',\dotsc,x_n'}} }_\infty
    \\
    & \leq \epsilon + L \norm{z - z'} + \epsilon \\
    & \leq 2\epsilon + \frac{DL}{n} ,
  \end{align*}
  so re-arranging gives
  \begin{equation*}
    (1-2\epsilon)n \leq DL .
  \end{equation*}
  An analogous argument handles the case when $\max\set{x_1,\dotsc,x_n} \neq \max\set{x_1',\dotsc,x_n'}$.
\end{proof}

\end{document}